\documentclass[11pt]{article} 

\usepackage[letterpaper,margin=1in]{geometry}

\usepackage[T1]{fontenc}
\usepackage{microtype}
\usepackage{times}
\usepackage[linesnumbered,ruled,vlined]{algorithm2e}
\usepackage[noend]{algpseudocode}
\usepackage{multirow}
\usepackage{bm}
\usepackage{diagbox,xcolor}
\usepackage{titlesec}
\titleformat*{\paragraph}{\bfseries}

\usepackage{amsthm,amsmath,amssymb,amsfonts,amssymb,mathtools}
\usepackage{amsmath,amssymb,amsfonts,amssymb,mathtools}
\usepackage{xcolor}
\usepackage{empheq}
\usepackage{dsfont}
\usepackage{mathrsfs}
\usepackage{graphicx}
\usepackage{enumitem}
\usepackage{verbatim}
\usepackage{aliascnt} 
\usepackage{xspace}
\usepackage{booktabs}
\usepackage{bbm}

\usepackage{caption}
\usepackage{tikz}
\usetikzlibrary{patterns}
\usetikzlibrary{arrows,shapes,automata,backgrounds,petri,positioning}
\usetikzlibrary{shadows}
\usetikzlibrary{calc}
\usetikzlibrary{spy}
\usetikzlibrary{angles, quotes}
\usetikzlibrary{matrix}
\usepackage{pgf,pgfplots}
\usepackage{pgfmath,pgffor}
\pgfplotsset{compat=1.17}

\usepackage{framed}

\definecolor[named]{ACMBlue}{cmyk}{1,0.1,0,0.1}
\definecolor[named]{ACMYellow}{cmyk}{0,0.16,1,0}
\definecolor[named]{ACMOrange}{cmyk}{0,0.42,1,0.01}
\definecolor[named]{ACMRed}{cmyk}{0,0.90,0.86,0}
\definecolor[named]{ACMLightBlue}{cmyk}{0.49,0.01,0,0}
\definecolor[named]{ACMGreen}{cmyk}{0.20,0,1,0.19}
\definecolor[named]{ACMPurple}{cmyk}{0.55,1,0,0.15}
\definecolor[named]{ACMDarkBlue}{cmyk}{1,0.58,0,0.21}

\usepackage[colorlinks,citecolor=blue,linkcolor=magenta,bookmarks=true]{hyperref}

\usepackage[nameinlink]{cleveref}
\crefname{ineq}{Inequality}{Inequality}
\creflabelformat{ineq}{#2{\upshape(#1)}#3}
\crefname{sub}{Subsection}{Subsection}
\creflabelformat{Subsection}{#2{\upshape(#1)}#3}
\crefname{sdp}{SDP}{SDP}
\creflabelformat{sdp}{#2{\upshape(#1)}#3}
\crefname{lp}{LP}{LP}
\creflabelformat{lp}{#2{\upshape(#1)}#3}
\usepackage{aliascnt}
\crefname{ineq}{Inequality}{Inequality}
\creflabelformat{ineq}{#2{\upshape(#1)}#3}
\crefname{sub}{Subsection}{Subsection}
\creflabelformat{Subsection}{#2{\upshape(#1)}#3}
\crefname{sdp}{SDP}{SDP}
\creflabelformat{sdp}{#2{\upshape(#1)}#3}
\crefname{lp}{LP}{LP}
\creflabelformat{lp}{#2{\upshape(#1)}#3}



\newcommand{\supp}{\mathsf{supp}}
\newcommand{\slr}{\mathsf{SLR}}
\newcommand{\noise}{\bm \xi}
\newcommand{\hyp}{{\widehat{\vec{w}}}}

\newcommand{\labels}{\vec{y}}

\newtheorem{theorem}{Theorem}[section]

\newtheorem{lemma}[theorem]{Lemma}
\newtheorem{informal theorem}[theorem]{Theorem (informal statement)}

\newtheorem{claim}[theorem]{Claim}

\newtheorem{remark}[theorem]{Remark}

\newtheorem{definition}[theorem]{Definition}

\newcommand\twonorm[1]{\|#1\|_2}
\newcommand\norm[1]{\left\| #1 \right\|}

\renewcommand\vec[1]{\mathbf{#1}}
\newcommand{\pr}{\mathbb{P}}
\DeclareMathOperator*{\E}{\mathbb{E}}

\def\multichoose#1#2{\ensuremath{\left(\kern-.3em\left(\genfrac{}{}{0pt}{}{#1}{#2}\right)\kern-.3em\right)}}

\newcommand{\R}{\mathbb{R}}

\newcommand{\eps}{\epsilon}

\newcommand{\poly}{\mathrm{poly}}

\newcommand{\x}{\vec x}
\newcommand{\vv}{\vec v}
\newcommand{\vu}{\vec u}
\newcommand{\vw}{\vec w}
\newcommand{\z}{\vec z}

\newcommand{\uopt}{\vec{u}^*}
\newcommand{\colspan}[1]{\mathsf{colspan}(#1)}
\newcommand{\vsigma}{\vec{\Sigma}}

\newcommand{\deltain}{\delta^{\mathrm{in}}}
\newcommand{\kappain}{\kappa^{\mathrm{in}}}

\newcommand{\normset}[1]{\mathcal{B}(#1)}

\newcommand{\iin}{I^{\mathrm{in}}}

\newcommand{\bout}{\vec{B}^{\mathrm{out}}}
\newcommand{\bin}{{\vec{B}^{\mathrm{in}}}}

\newcommand{\btemp}{\vec{B}^{\mathrm{temp}}}
\newcommand{\itemp}{I^{\mathrm{temp}}}
\newcommand{\ztemp}{\vec{z}^{\mathrm{temp}}}
\newcommand{\iout}{I^{\mathrm{out}}}

\newcommand{\improvenorm}{\mathsf{ImproveNorm}}
\newcommand{\Gauss}{\mathcal{N}}

\newcommand{\datadistribution}{D}

\newcommand{\wopt}{{\vw}^*}

\newcommand{\dataset}{\mathbf{X}}

\newcommand{\inv}{\mathsf{INV}}
\title{Sparse Linear Regression is Easy on Random Supports}
\date{}
\author{Gautam Chandrasekaran\thanks{\texttt{gautamc@cs.utexas.edu}. Supported by the NSF AI Institute for Foundations of Machine Learning ( IFML).} \\
	 UT Austin
     \and
     Raghu Meka\thanks{\texttt{raghum@cs.ucla.edu}. Supported by the NSF Award CCF-2217033 (EnCORE: Institute for Emerging CORE Methods in Data Science).}\\
     UCLA
     \and 
     Konstantinos Stavropoulos\thanks{\texttt{kstavrop@utexas.edu}. Supported by the NSF AI Institute for Foundations of Machine Learning (IFML), and by the Apple Scholars in AI/ML PhD fellowship.}\\
     UT Austin
    }
\begin{document}
\pagenumbering{gobble} 
\maketitle{}
\begin{abstract}
Sparse linear regression is one of the most basic questions in machine learning and statistics. Here, we are given as input a \emph{design matrix} $\vec{X} \in \R^{N \times d}$ and \emph{measurements} or \emph{labels} $\vec{y} \in \R^N$ where $\vec{y} = \vec{X} \vec{w}^* + \bm \xi$, and $\bm \xi$ is the noise in the measurements. Importantly, we have the additional constraint that the unknown \emph{signal vector} $\vec{w}^*$ is sparse: it has $k$ non-zero entries where $k$ is much smaller than the ambient dimension. Our goal is to output a prediction vector $\widehat{\vec{w}}$ that has small prediction error: $\frac{1}{N}\cdot \|\dataset \vec{w}^* - \dataset \widehat{\vec{w}}\|^2_2$. 

Information-theoretically, we know what is best possible in terms of measurements: under most natural noise distributions, we can get  prediction error at most $\epsilon$ with roughly $N = O(k \log d/\eps)$ samples. Computationally, this currently needs $d^{\Omega(k)}$ run-time. Alternately, with $N = O(d)$, we can get polynomial-time. Thus, there is an exponential gap (in the dependence on $d$) between the two and we do not know if it is possible to get $d^{o(k)}$ run-time and $o(d)$ samples. 

We give the first generic positive result for worst-case design matrices $\vec{X}$: For any $\vec{X}$, we show that if the support of $\vec{w}^*$ is chosen at random, we can get prediction error $\epsilon$ with  $N = \poly(k, \log d, 1/\eps)$ samples and run-time $\poly(d,N)$. This run-time holds for any design matrix $\vec{X}$ with condition number up to $2^{\poly(d)}$. 

Previously, such results were known for worst-case $\vec{w}^*$, but only for random design matrices from well-behaved families, matrices that have a very low condition number ($\poly(\log d)$; e.g., as studied in compressed sensing), or those with special structural properties.

\end{abstract}

\newpage
\pagenumbering{arabic}   
\setcounter{page}{1}      
\section{Introduction}
The problem of \emph{sparse-linear regression} (SLR) is one of the most basic questions in machine learning and statistics. 
Formally, we are given a \textit{design matrix} $\dataset\in \R^{N\times d}$ and labels $\labels\in \R^{N}$ where 
\[
\labels=\dataset\wopt+\noise.
\] Here, $\wopt\in \R^{d}$ is a sparse vector: it only has $k \ll d$ non-zero entries, and $\noise$ denotes noise in the observations. Throughout this paper, we assume that $\noise$ is a mean zero $\sigma$-sub-gaussian random vector that is independent of $\dataset$. The goal is to output a vector $\widehat{\vw}\in \R^{d}$ such that the \textit{prediction error} 
\[
\frac{1}{N}\twonorm{\dataset\wopt-\dataset\widehat{\vw}}^2
\]
is small. This problem has seen extensive theoretical study \cite{tibshirani1996regression,candes2005decoding,candes2007dantzig,bickel2009simultaneous,van2009conditions,kelner2022power} as well as numerous practical applications \cite{levy1981reconstruction,santosa1986linear,wu2009genome,fan2011sparse,rish2014practical}.
\paragraph{Computational-Statistical gap} The main objects of study for us are the \textit{statistical} and \textit{computational} complexities of the problem. The statistical complexity (number of \emph{measurements} or rows in $\dataset$) of the problem is mostly settled: it is well known \cite{raskutti2011minimax} that one can achieve $\epsilon$ prediction error with sample complexity $N=O(\sigma k\log(d)/\epsilon)$. The computational complexity, on the other hand is surprisingly still wide open. This is one of the few problems where there is currently an exponential gap between the \emph{statistical threshold} (the number of samples needed to get small error information-theoretically) and the \emph{computational threshold} (the number of samples needed to get small error with polynomial-time algorithms). 

The natural algorithm that enumerates all subsets of size at most~$k$ hits the optimal sample complexity but takes time about~$d^{k}$, which is super-polynomial for any $k=\omega(1)$. With $N=O(d)$ samples, plain least-squares works. What’s unknown is whether we can do better: can we achieve prediction error~$0.1$ (for constant $\sigma$) with $N=o(d)$ samples and runtime~$d^{o(k)}$? Even more loosely, is there any polynomial-time algorithm that achieves error~$0.1$ with $N=o(d)$ for some $k=\omega(1)$, or is this computationally hard under natural assumptions?

\paragraph{Current lower bounds} There is partial evidence to suggest that polynomial time algorithms for sparse linear regression need $\Omega(d)$ samples. In particular, in the fixed design setting, it is known that the task of finding a sparse vector achieving low prediction error is computationally hard \cite{natarajan1995sparse,zhang2014lower,foster2015variable,gupte2021fine,gupte2024sparse}. This captures a subset of possible algorithms known as \textit{proper} learners. For \textit{improper} learners and for settings with (possibly correlated) random designs, existing lower bounds remain limited: they either make (1) very restrictive assumptions on the algorithm \cite{zhang2017optimal,kkmr_lb}, or (2) only give weak sample complexity lower bounds \cite{buhai2024computational,kelner2024lasso} (they only rule out sample complexity sub-quadratic in $k$). The latter two lower bounds are conditional, they assume variants of the low degree conjecture \cite{hopkins2018statistical}. This scarcity of lower bounds against more general classes of algorithms is in stark contrast to other well studied problems in statistical inference like planted clique, where there are lower bounds known against strong classes of algorithms such as Statistical Query (SQ) \cite{feldman2017statistical} or Sum of Squares (SoS) hierarchies \cite{barak2019nearly}. In fact, for sparse linear regression in the random-design setting, there is not even a widely accepted conjectured hard instance.

\paragraph{Known algorithms} On the algorithmic side, all polynomial time algorithms that achieve sample complexity $o(d)$ make additional assumptions on the design matrix $\dataset$. Most of the classic approaches such as the celebrated \emph{Lasso} estimator \cite{tibshirani1996regression} and the Dantzig selector \cite{candes2007dantzig} assume that $\dataset$ satisfies some form of well-conditioning. The highly-influential works on \emph{compressed sensing} show how to solve sparse-linear regression efficiently when $\vec{X}$ is well-conditioned in some quantitative ways such as  \textit{incoherence} \cite{donoho1989uncertainty}, the \textit{restricted isometry property} \cite{candes2005decoding}, the \textit{restricted eigen value condition} \cite{bickel2009simultaneous}, the \textit{compatibility condition} \cite{van2009conditions} etc. More recent work has considered ways to relax these well-conditioning assumptions by studying alternate structural properties of the covariance matrices of the data, such as having only few outlier eigenvalues or \emph{precision matrices} (inverse of the covariance matrix) with low tree-width \cite{kelner2022power,kelner2023feature,kelner2024lasso}.

A common theme in all the above results is that they make (often quite strong) assumptions on the design matrix $\dataset$, and give efficient algorithms that achieve low prediction error for all sparse regression vectors. We argue that making such strong assumptions on the design matrix is undesirable, as the design matrix is not something that the algorithm designer has control over. We flip the script on this theme: we make no assumption on the design matrix\footnote{we only have a doubly logarithmic dependence on the condition number, which remains logarithmic in $d$ when the bit complexity of $\dataset$ is $\poly(d)$.}, but assume that the true sparse regression vector is not worst-case. In particular, we study the case where the sparse support of the regression vector is chosen uniformly at random.

\subsection*{Our results}

Our main result shows that the widely-believed exponential computational-statistical gap for sparse linear regression vanishes under a natural average case model, where the support of the regression vector is chosen uniformly at random. 
Before stating our result formally, we define the sparse condition number of a matrix.
\begin{definition}[$k$-sparse condition number of a dataset]
\label{defn:k_sparse_cd}
  Given a dataset $\dataset\in \R^{N\times d}$, the $k$-sparse condition number  $\kappa(\dataset)$ is defined as 
\begin{equation}
    \kappa(\dataset)\coloneq \max_{\substack{\norm{\vw}_0\leq k\\\norm{\dataset\cdot \vw}_2=\sqrt{N}}}\twonorm{\vw}
\end{equation}  
\end{definition}

Now, we are ready to state our main result saying that we can efficiently get prediction error $\eps$ with sample complexity $\poly(k, \log \log \kappa(X), \log d, 1/\eps)$ for random supports and sub-Gaussian noise distributions.

\begin{theorem}
\label{thm:low_training_error}
There exists a randomized algorithm $\mathcal{A}$ such that for any $\dataset$ satisfying $\max_{i\in [d]}\twonorm{\dataset^{(i)}}\leq \sqrt{N}$, the following holds: if $S$ is drawn uniformly from $[d]$ with $|S|=k$, $\wopt$ is any vector supported on $S$ with $\twonorm{\dataset\wopt}\leq\sqrt{N}$, and  $\labels=\dataset \vw^*+\noise$  with mean zero $\sigma$-sub-gaussian $\noise$ independent of $\dataset$, then with probability at least $1-\delta$ over the support of $S$, and with probability $0.9$ over the noise and randomness of the algorithm, the vector $\widehat{\vw}=\mathcal{A}(\dataset,\labels)$ satisfies
  \[
   \frac{1}{N}\cdot \norm{\dataset\cdot (\wopt-\hyp)}_2^2\leq \frac{C\sigma k(\log \log \kappa(\dataset))}{\sqrt{N}} \cdot \big(\frac{k}{\sqrt{\delta}}+\sqrt{\log d}\big)
   \] for universal constant $C$. Furthermore, $\mathcal{A}(\dataset,\labels)$ runs in time $\poly(N,d,\log \kappa(\dataset),1/\delta)$.
\end{theorem}
Note that without the requirement of efficiency, the best possible error-rate, that holds even for arbitrary sparse $w^*$, is $O(\sigma k (\log d)/N)$. Our error-rate matches this essentially up to polynomial factors for random supports.

In the above theorem, we achieve prediction error $\epsilon$ for \[N=O(\sigma^2k^4(\log \log \kappa(\dataset))^2\log (d)/(\epsilon^2\delta)).\] This doubly logarithmic dependence on the condition number $\kappa(\dataset)$ is a substantial improvement over all prior polynomial time algorithms and is mostly benign - in particular, we should expect it to be only $\poly(\log d)$ for most interesting cases (e.g., where the bit-complexity of $\dataset$ is bounded). The sample complexities of previous algorithms scale polynomially with condition number $\kappa(\dataset)$ (without imposing additional assumptions on $\dataset$). 

It is perhaps surprising that such a mild dependence on condition number is even possible. From an optimistic perspective, it provides evidence on the computational tractability of the problem well beyond the well-conditioned regime (at least for non-adversarial supports). The result also informs the search for lower bound instances; it  suggests that any lower-bound for the problem must involve a sparse regression vector that is adversarially coupled with the design matrix.
\begin{remark}\label{remark:delta-dependence}
Our polynomial dependence on $1/\delta$ in the runtime is probably tight. This is because any improved dependence scaling with $(1/\delta)^{o(1)}$ would imply a worst case SLR algorithm running in polynomial time for some $k=\omega(1)$, by setting $\delta=(1/d^k)$. This is believed to not be possible.    
\end{remark}
As a straightforward consequence of our techniques, we also get learnability in the random design setting, where the marginal distribution is $\Gauss(\vec{0},\vsigma)$ for highly ill-conditioned $\vsigma$. This has been a well studied problem in the literature (see \cite{kelner2022power} and references therein) and no positive results were known for general ill conditioned $\vsigma$ prior to our work. Before stating our result, we need a definition. 

\begin{definition}
   For positive definite $\vsigma$, $k$-sparse $\wopt$ and positive $\sigma$, we define the distribution $\slr(\vsigma,\wopt,\sigma)$ on $\R^{d}\times \R$ as the following: sample $\vec{x}\sim \Gauss(\vec{0},\vsigma)$ and $y= \wopt\cdot\vec{x}+\xi$ and output $(\vec{x},y)$. The noise $\xi$ is mean zero, $\sigma$-sub-gaussian and independent of $\vec{x}$.
\end{definition} 
For a positive semi-definite matrix $\vsigma\in \R^{d\times d}$, define $\kappa(\vsigma)\coloneq \max_{\norm{\vw}_{0}\leq k}\frac{\twonorm{\vw}}{\twonorm{\vsigma^{\frac{1}{2}}\vw}}$.
We are now ready to state our theorem on sparse linear regression over random designs with entries drawn from a correlated Gaussian. We give a polynomial time algorithm that gets prediction error $\epsilon$ with sample complexity $\poly(k,\log \log \kappa(\vsigma),\log d,1/\epsilon)$ for random supports and sub-Gaussian noise.
\begin{theorem}
\label{thm:low_population_error}
There exists a randomized algorithm such that for any positive semi-definite matrix $\vsigma$ with $\max_{i\in [d]}\vsigma_{ii}\leq 1$, the follow holds: if $S$ is a uniformly random subset of $[d]$ with size $|S|=k$, $\wopt$ is any vector supported on $S$ with $\E_{\vec{x}\sim \Gauss(\vec{0},\vsigma)}[(\wopt\cdot \vec{x})^2]\leq 1$, then the algorithm, after drawing $N=O(k^4(\log \log \kappa(\vsigma))^2\log (d)/(\epsilon^2\delta))$ i.i.d. samples from $\slr(\vsigma,\wopt,\sigma)$, outputs $\widehat{\vw}\in \R^d$ such that 
   \[ 
    \E_{\vec{x}\sim \Gauss(\vec{0},\vsigma)}[(\hyp\cdot \vec{x}-\wopt\cdot \vec{x})^2]\leq \epsilon
   \] with probability at least $1-\delta$ over the support $S$ of $\wopt$, and with probability at least $0.9$ over the samples and randomness of the algorithm. Furthermore, the algorithm runs in time $\poly(N,d,\log \kappa(\vsigma),1/\delta)$.
 \end{theorem}  
One important comment about the strength of the above theorem is that the algorithm does not require knowledge of the covariance $\vsigma$.

\paragraph{Probability of Success} In the statements of our main results (\Cref{thm:low_training_error,thm:low_population_error}) we distinguish between (1) the probability of failure $\delta$ due to the random choice of the support of $\vw^*$ and (2) the probability of failure $\gamma$ due to the randomness of the algorithm, together with the randomness of the noise $\noise$. This distinction is important, as the dependence of our runtime on $\delta$ cannot be significantly improved without solving a worst-case problem (see \Cref{remark:delta-dependence}). On the other hand, we can boost the success probability (over randomness of the algorithm and label noise) to $1-\gamma$ by paying a multiplicative $\log(1/\gamma)$ factor in the runtime and sample complexity. This is because the success probabilities of all the technical ingredients of our algorithm can be amplified (see \Cref{thm:improve_norm_single_step,thm:good_basis_implies_low_error,theorem:generalization}). 

\section{Notation}
 We use capitalized boldened letters such as $\bf X,\bf Y$ to denote matrices. We use uncapitalized bold letters to denote vectors. For a matrix $\vec{X}$, we use $\vec{X}^i$ to denote the $i$-{th} power, $\vec{X}^{(i)}$ to denote the $i$-th column and $\vec{X}_{i}$ to denote the $i$-th row. We use square-bracket superscripts (e.g., $\vec{X}^{[i]}$) to denote the $i$-th element of a sequence. For example, $\vec{X}^{[1]},\ldots ,\vec{X}^{[t]}$ and $\alpha^{[1]},\ldots ,\alpha^{[t]}$ denote sequences of matrices and scalars. Given a set $S$, we denote by $\vec{X}_S$ the sub-matrix of $\vec{X}$ formed by the columns with indices in $S$. Similarly, for a vector $\vw$, we denote by $\vw_{S}$ the vector formed by the indices in $S$. $\vec{I}_{(d)}$ and $\vec{J}_{(d)}$ represent the identity matrix and all ones matrix of dimension $d$ respectively.

 For a vector $\vw$, we use $\supp(\vw)$ to denote the set of indices of $\vw$ that are non-zero. We denote the column span of a matrix $\vec{X}$ by $\colspan{\vec{X}}$. We denote with $\vec X^\top$ the transpose of matrix $\vec X$. For an invertible matrix $\vec B$, we denote with $\vec B^{-\top}$ the transpose of its inverse, i.e., $\vec B^{-\top} = ({\vec B}^{-1})^\top$.
 
 We say that a random variable $X\in \R$ is $\sigma$-sub-gaussian if $\Pr[|X|>t]\leq e^{-t^2/(2\sigma^2)}$. We say that a random vector $\vec{x}\in \R^{d}$ is $\sigma$-sub-gaussian if the following holds for all $\vec{w}\in \R^{d}$ with $\twonorm{\vw}\leq 1$: ($\vw\cdot \vec{x}$) is $\sigma$-sub-gaussian. 

 We use $\binom{[d]}{k}$ to denote the set of subsets of $[d]$ of size $k$. Given any set $S$, we use $X\sim S$ to denote random variable corresponding to the uniform distribution over the set.
\section{Proof Overview}
In this section, we present the overview of our proof. We recall some notation. Throughout this section $\dataset\in \R^{N\times d}$ is the input design matrix and $\wopt$ is the ground truth regression vector. We assume that $\max_{i\in [d]}\twonorm{\dataset^{(i)}}\leq \sqrt{N}$ and $\frac{1}{N}\twonorm{\dataset\wopt}^2\leq 1$. The input to the algorithm is $(\dataset,\labels)$ where $\labels=\dataset\wopt+\noise$. $\noise$ is a mean zero $\sigma$-sub-gaussian random vector independent of $\dataset$. For a set $S\subseteq [d]$, we define $\normset{S}$ as the set of vectors $\vw\in \R^{d}$ such that $\supp(\vw)\subseteq S$ and $\frac{1}{N}\cdot \twonorm{\dataset\vw}^2\leq 1$.
\subsection{Easy Cases for Sparse Linear Regression}
\label{sec:easy_cases}
First, we will overview two settings where we know polynomial time algorithms for sparse regression that achieve low prediction error with $N=o(d)$ samples. 
\paragraph{Known Support}
Perhaps the simplest instance of the sparse regression problem is when one knows a small set $S$ that contains the support of $\wopt$. In this case, it follows from standard arguments that ordinary least squares with support constrained to be in $S$ achieves prediction error $\epsilon$ for $N\ge N_{\mathrm{known}} = O(\sigma|S|/\epsilon)$ (e.g., see Theorem~2.2 from \cite{rigollet2023highdimensionalstatistics}).

\paragraph{Ground Truth has Low Norm}
Another case where we can efficiently solve sparse regression with low sample complexity is when $\norm{\wopt}_1$ is small. The classic "slow rate" analysis of the Lasso (see, for example, Theorem~2.15 from \cite{rigollet2023highdimensionalstatistics}, or Theorem~7.20 from \cite{wainwright2019high}) implies that the Lasso algorithm (constrained or penalty form) achieves prediction error $\epsilon$ whenever $N\ge N_1 = O(\sigma^2\norm{\wopt}_1^2\log (d)/\epsilon^2)$ and $\max_{i\in [d]}\twonorm{\dataset^{(i)}}\leq \sqrt{N}$.

The above result can be used to guarantee low prediction error for sparse regression when $\kappa(\dataset)$ is bounded. Consider $\wopt$ for which $\frac{1}{N}\twonorm{\dataset\wopt}^2\leq 1$. From \Cref{defn:k_sparse_cd}, we have that $\twonorm{\wopt}\leq \kappa(\dataset)$. This implies that $\norm{\wopt}_1\leq \sqrt{k}\cdot \kappa(\dataset)$. Thus, for $\wopt$ satisfying $\frac{1}{N}\twonorm{\dataset\wopt}^2\leq 1$, we have that Lasso achieves prediction error $\epsilon$ when $N\ge N_{\kappa}=O(\sigma^2\kappa(\dataset)^2k\log (d)/\epsilon^2)$. This is the reason why achieving low prediction error is easy for well-conditioned design matrices. 

\subsection{Good Preconditioners and Sparse Linear Regression}
As discussed in the introduction, the design matrices we consider may have $\kappa(\dataset)$ scaling as $\poly(d)$ (or even $e^{{\poly(d)}})$. In such cases, the sample complexity bounds that scale polynomially in $\kappa(\dataset)$ (see \Cref{sec:easy_cases}) become vacuous, as a direct analysis of ordinary least squares already yields a sample complexity of $O(d/\epsilon)$. 

Recall that our aim is to achieve low sample complexity for regression when the support sets are uniform random sparse sets. Thus, a natural first attempt would be to argue that $\kappa(\dataset_S)$ is small with probability at least $1-\delta$ over $S\sim \binom{[d]}{k}$. This would imply that ground truth vectors supported on such $S$ would have small norm with high probability which would allow us to apply the "slow rate" bound from \Cref{sec:easy_cases}. Unfortunately, this is too good to be true for arbitrary datasets $\dataset$: there are simple examples of matrices $\dataset$ such that $\dataset_{S}$ is ill conditioned with high probability over random sets $S$ (an example would be the matrix $\vec{J}_{(d)}+\epsilon\cdot  \vec{I}_{(d)}$ for small $\epsilon$). 

Although this natural first attempt fails, we show that we do not need to abandon it completely: all we need is a basis transformation. 
We show that, after applying an appropriate basis transformation to the dataset $\dataset$, there exists a fixed small set $I \subseteq [d]$, depending only on the data, such that, with high probability over the random support, the ground truth linear function admits a representation in the new basis whose $\ell_1$-norm outside $I$ is $O(k)$. Moreover, the set $I$ of (potentially) large coordinates can be identified efficiently. 
The transformed representation of the ground truth can thus be viewed as satisfying a hybrid of the two easy cases discussed in \Cref{sec:easy_cases} --- a known set where coefficients may be large, and a bounded $\ell_1$-norm outside this set. We will use this structure to design our final regression algorithm. 
Formally, we define the following notion of a good preconditioner. 
\begin{definition}[Good Preconditioner]
\label{defn:good_change_of_basis}
    Let $\datadistribution$ be a distribution on $\R^{d}$ with $\max_{i\in [d]}(\E_{\vec{x}\sim \datadistribution}[\vec{x}_i^2])\leq 1$. A pair $(\vec{B},I)$ where $\vec{B}\in \R^{d\times d}$ is invertible and $I\subseteq [d]$ is called an $(\ell,k,\delta,\kappa)$-good preconditioner  for $\datadistribution$ if the following hold: 
    \begin{enumerate}
        \item $\max_{i\in [d]}\E_{\vec{x}\sim \datadistribution}[(\vec{B}\vec{x})_i^2]\leq 1$,
        \item $|I|\leq \ell$ and $\supp((\vec{B}^{-1})_i)\subseteq I\cup \{i\}$ for all $i\in [d]$.
        \item With probability at least $1-\delta$ over $S\sim \binom{[d]}{k}$, for all $\vec{w}\in \R^{d}$ with $\supp(\vw)\subseteq S$ and $\E_{\vec{x}\sim \datadistribution}[(\vw\cdot \vec{x})^2]\leq 1$, it holds that \[\norm{(\vec{B}^{-\top}\vw)_{\overline{I}}}_{\infty}\leq \kappa.\]
    \end{enumerate}
We drop the last parameter when $\kappa=O(1)$. That is, we say that $(\vec{B},I)$ is an $(\ell,k,\delta)$-good preconditioner if it is an $(\ell,k,\delta,O(1))$-good preconditioner. 
\end{definition}

To interpret the above notion of a good preconditioner, it is helpful to consider the case where $D$ is the uniform distribution over the rows of the dataset $\dataset$. Let $\vec{Z}=\dataset\vec{B}^{\top}$ be the dataset after the change of basis. Let $\uopt$ be the vector $\uopt\coloneq\vec{B}^{-\top}\wopt$. This is the representation of the ground truth vector over the new basis. Clearly, we have that $\vec{Z}\uopt=\dataset\wopt$.  We now explain the significance of the three properties in \Cref{defn:good_change_of_basis}:
\begin{itemize}
    \item \textbf{Property 1}. This property controls the variance of the coordinates of the output basis. 
    In our case, it implies that $\max_{i\in [d]}\twonorm{\vec{Z}^{(i)}}\leq O(\sqrt{N})$. This will be important for our final regression algorithm (analogous to the column-norm condition for the "slow rate" bound from \Cref{sec:easy_cases}). 
    \item \textbf{Property 2}. This property controls the number of indices of $\uopt$ that are large. This will be crucial for our final sample complexity bound. The property that $(\vec{B}^{-1})_{i}\subseteq I\cup \{i\}$ for all $i\in [d]$ implies that $\supp(\uopt)\subseteq I\cup \supp(\wopt)$. 
    \item \textbf{Property 3}. This property implies that $\|{\uopt_{\overline{I}}}\|_{\infty}\leq \kappa$ with probability at least $1-\delta$ over the support of $\wopt$. Recall that property (2) implied that $\supp(\uopt)\subseteq I\cup \supp(\wopt)$. Thus, we have that $\|{\uopt_{\overline{I}}}\|_{1}\leq k\cdot \kappa$ as $|\supp(\wopt)|\leq k$. In particular, when $\kappa=O(1)$, this implies that $\|{\uopt_{\overline{I}}}\|_{1}\leq O(k)$ which is the bound we need for our final regression algorithm. Note that, without property (1), satisfying property (3) would trivial, as one could rescale each column of $\dataset$ arbitrary to inflate their magnitude, resulting in decreased coefficients for $\uopt$.
\end{itemize}

Property (3) in the above definition holds with probability $1-\delta$ over the random support of the vector $\wopt$. The size of the set $I$ that we construct will depend inversely polynomially on the failure probability $\delta$. 

The high-level framework of using a \emph{change-of-basis} matrix to improve the sample complexity of SLR is not new to our work. There are many examples of hard instances for the Lasso becoming tractable after a change of basis \cite{foygelfast, dalalyan2017prediction, zhang2017optimal, kelner2019learning}. A recent line of work on Preconditioned Lasso \cite{kelner2022power, kelner2023feature, kelner2024lasso} has identified key structural properties of the design matrix that enable such helpful change-of-basis transformations.

Our contribution is conceptually very different. All prior work on preconditioned methods assumes structural properties of the design matrix that allow them to construct good preconditioners. We make no such assumptions: our algorithm constructs a preconditioner for \emph{any} design matrix, and this preconditioner improves performance with good probability over random supports. Our main technical tool is an efficient algorithm that finds a good preconditioner given the dataset~$\dataset$, as stated in the following theorem.

\begin{theorem}[Finding a good preconditioner]
\label{thm:find_good_basis}
    Let $\dataset\in \R^{N\times d}$ be a dataset with $\max_{i\in [d]}\twonorm{\dataset^{(i)}}\leq \sqrt{N}$. There is an algorithm that runs in time $\poly(N,d,1/\delta,\log \kappa(\dataset))$ and with probability at least $0.99$ outputs an invertible matrix $\vec{B}\in \R^{d\times d}$ and set $I\subseteq [d]$ such that $(\vec{B},I)$ is a $(\ell,k,\delta)$-good preconditioner for the uniform distribution over the rows of $\vec{X}$ with $\ell = O({k^2(\log \log\kappa(\dataset))^2}/{
    \delta})$. 
\end{theorem}

The bound of \Cref{thm:find_good_basis} is quite sharp: the size of the output set $I$ has a doubly logarithmic dependence on the condition number. This will play a crucial role in our final sample complexity bound. See \Cref{sec:good_precond_overview} for an overview of the proof of the above theorem. 
\begin{remark}[Lower bounds for good preconditioners]
    We show in \Cref{thm:good_precond_lower_bound} that there are simple distributions (even of the form $\Gauss(0,\vsigma)$) for which any $(\ell,k,\delta)$-good preconditioner must have $\ell\geq \Omega(k^2/\delta)$. This shows that except for the mild dependence on the condition number, \Cref{thm:find_good_basis} is tight.
\end{remark}

Once we have the good preconditioner $(\vec{B},I)$ guaranteed by the above theorem, we will then solve the regression task over the new basis $\vec{Z}=\dataset\vec{B}^{\top}$. We will show the following theorem.

\begin{theorem}
    \label{thm:good_basis_implies_low_error}
   Let $\dataset\in \R^{N\times d}$ be a dataset with $\max_{i\in [d]}\twonorm{\dataset^{(i)}}\leq \sqrt{N}$. Let $(\vec{B},I)$ be an $(\ell,k,\delta)$-good preconditioner matrix for the uniform distribution over the rows of $\vec{X}$. Then, there exists an algorithm that runs in polynomial time such that with probability at least $1-\delta$ over $S\sim \binom{[d]}{k}$ and probability $0.9$ over the randomness of the algorithm, for any $\wopt\in \R^{d}$ with $\supp(\wopt)=S$ and $\twonorm{\dataset\wopt}\leq \sqrt{N}$, outputs a vector $\widehat{\vw}$ such that 
   \[
   \frac{1}{N}\cdot \norm{\dataset (\wopt-\hyp)}_2^2\leq 
   C\sigma k \cdot \left(\frac{\sqrt{\ell}+\sqrt{\log d}}{\sqrt{N}}\right)
   \] for universal constant $C$. 
\end{theorem}
\begin{proof}[Proof Sketch]
    Construct the new dataset $\vec{Z}=\dataset\vec{B}^{\top}$. Observe that $\uopt=\vec{B}^{-\top}\wopt$ satisfies $\|{\uopt_{\overline{I}}}\|_1\leq Ck$ for some constant $C$ with probability at least $1-\delta$ over the support of $\wopt$. Run least squares regression over $(\vec{Z},\labels)$ with a constraint that $\ell_1$-norm of the predictor outside the set $I$ is at most $Ck$.\footnote{We have an extra constraint $\twonorm{\dataset\widehat{\vw}}\leq \sqrt{N}$ for technical reasons.} This is a convex program that can be efficiently solved (recall that we know the set $I$). Let the solution be $\widehat{\vu}$. We show that $\frac{1}{N}\twonorm{\vec{Z}\uopt-\vec{Z}\widehat{\vu}}^2$ is at most $O\big(\sigma k\cdot \frac{\sqrt{\ell}+\sqrt{\log d}}{\sqrt{N}}\big)$ using an analysis that is a hybrid of the proofs of the two cases in \Cref{sec:easy_cases}. The output of our algorithm is $\widehat{\vw}\coloneq \vec{B}^{\top}\widehat{\vu}$.
\end{proof}
The full proof of the above theorem can be found in \Cref{proof:good_basis_implies_low_error}. 
We are now ready to prove our main theorem.

\begin{proof}[Proof of \Cref{thm:low_training_error}]
    First, run the algorithm from \Cref{thm:find_good_basis} with input $\dataset$. Let $(\vec{B},I)$ be the output of this step. Now, run the algorithm from \Cref{thm:good_basis_implies_low_error} with $\ell=O(k^2(\log \log \kappa(\dataset))^2/\delta)$. The final result follows from the prediction error guarantee of \Cref{thm:good_basis_implies_low_error}. 
\end{proof}

\subsection{Constructing a Good Preconditioner}
\label{sec:good_precond_overview}
We will now overview our algorithm to find a good preconditioner. We will sketch the proof of \Cref{thm:find_good_basis}.  This step constitutes a bulk of our technical work. We refer the reader to \Cref{sec:good_precond_details} for the detailed proof. In the current section, our aim is to convey the main ideas and hence we omit some technical details. In particular, we will only focus on property (3) of \Cref{defn:good_change_of_basis} and take the first two for granted (they will hold true by construction, see \Cref{proof:improve_norm_single_step} for more details). Recall that property (3) relates to the magnitude of the coefficients of the ground truth vector when written in the new basis (after application of the preconditioner).

\paragraph{A win-win analysis} We first present a weaker result: we will efficiently find an $(O(k^2/\delta), k, \delta, \sqrt{\kappa(\dataset)})$-good preconditioner for the uniform distribution over the rows of $\dataset$. Note that this weaker result combined with \Cref{thm:good_basis_implies_low_error} already gives us a sample complexity bound that has a square root improvement (over the "slow rate" bound from \Cref{sec:easy_cases}) in the dependence on $\kappa(\dataset)$. Thus, if $\kappa(\dataset)$ was $o(d^2)$, we already achieve $o(d)$ sample complexity for regression over random supports.

We implement a win-win argument. Given a candidate preconditioner $(\vec{B},I)$, we show that one of following two cases must hold: (a) $(\vec{B},I)$ is the good preconditioner we desire, or (b) $\dataset$ must have some additional structure that can be exploited to improve the quality of the preconditioner.

To build intuition, we start with the trivial preconditioner $(\vec{I}_{(d)},\emptyset)$. Case (a) corresponds to $(\vec{I}_{(d)},\emptyset)$ being a $(0,k,\delta,\sqrt{\kappa(\dataset)})$-good preconditioner for the uniform distribution over the rows of $\dataset$. This implies that for random subset $S\sim\binom{[d]}{k}$, it holds that $\norm{\vw}_{\infty}\leq \sqrt{\kappa(\dataset)}$ for all $\vw\in \normset{S}$ with probability at least $1-\delta$ over $S$. Recall that $\normset{S}$ is the set $\{\vw\mid \twonorm{\dataset\vw}^2\leq N \text{ and }\supp(\vw)\subseteq S\}$. Conversely, in case (b) it must hold that with probability at least $\delta$ over $S\sim \binom{[d]}{k}$, $\max_{\vw\in \normset{S}}\norm{\vw}_{\infty}>\sqrt{\kappa(\dataset)}$.

We will show that in the latter case, $\dataset$ satisfies a useful structural property that we can exploit: there exists a set of $k-1$ columns of $\dataset$ whose span well approximates an $\Omega(\delta/k)$ fraction of the other columns of $\dataset$. We will soon formalize what the word "approximates" means in the previous sentence. 

First, we give some intuition as to why a structure of this form must hold for $\dataset$. When $\max_{\vw\in \normset{S}}\norm{\vw}_{\infty}$ is large, it implies that there are high norm vectors $\vw$ with $\supp(\vw)\subseteq S$ such that $\twonorm{\dataset\vw}\leq \sqrt{N}$.

\paragraph{Warm-up: Random sub-matrices are singular} We start with the limiting case when $\max_{\vw\in \normset{S}}\norm{\vw}_{\infty}$ tends to infinity: this implies that $\dataset_S$ is singular. Thus, in this extreme case, we have the property that $\dataset_S$ is singular with probability at least $\delta$ over $S\sim \binom{[d]}{k}$. For such $\dataset$, we show that there exists a set of $k-1$ columns of $\dataset$ such that an $\Omega(\delta/k)$-fraction of the remain columns lie in their span. Formally, we state the following lemma (we could not find this lemma stated explicitly in the literature; however, its proof is elementary and it is likely folklore).
\begin{lemma}
\label{lem:find_basis_random_low_rank}
    Let $\dataset$ be a matrix with $d$ columns. Suppose $\dataset_{S}$ is singular with probability at least $\delta$ over $S\sim \binom{[d]}{k}$. Then there exists  a set $S'$ of size $k-1$ such that $\colspan{\dataset_{S'}}$ contains an $\Omega(\delta/k)$ fraction of the columns of $\dataset$.
\end{lemma}
\begin{proof}
    The proof is by an averaging argument. For each set $S$ for which $\dataset_S$ is singular, let $\vw(S)$ be an explicit non-zero vector with support in $S$ such that $\dataset\vw(S)=\vec{0}$. Consider the random variable $(S,i)$ where $S\sim \binom{[d]}{k}$ and $i\sim S$. From the premise, we have that $\dataset_S$ is singular with probability at least $\delta$. Now, since $\vw(S)$ is non-zero and $|S|=k$, we have that with probability at least $\delta/k$ over $(S,i)$ defined earlier, it holds that $\dataset_S$ is singular and $\vw(S)_{i}\neq 0$.
    
    We now define an alterate way to sample $(S,i)$: sample $S'\sim \binom{[d]}{k-1}$, $i\sim [d]\setminus S'$ and form $(S,i)=(S'\cup \{i\},i)$. Clearly, these two sampling techniques result in the same distribution. We will use the latter sampling technique. From an averaging argument, we have that with probability at least $\delta/(2k)$ over $S'\sim \binom{[d]}{k-1}$ and probability at least $\delta/(2k)$ over $i\sim [d]\setminus S'$, it holds that $\dataset_{S'\cup \{i\}}\vw(S'\cup\{i\})=\vec{0}$ and $\vw(S'\cup \{i\})_{i}\neq 0$. This implies that $\dataset_{i}$ is in the span of $\dataset_{S'}$ with probability at least $\delta/(2k)$ over $i\sim [d]\setminus S'$. Thus, $S'$ is the set of $k-1$ indices in the lemma statement. 
\end{proof}

\paragraph{Random sub-matrices are ill-conditioned} In the previous lemma, we showed a strong structural property of the dataset $\dataset$ for the extreme case when $\max_{\vw\in \normset{S}}\norm{\vw}_{\infty}$ tends to infinity. We are ready now ready state an approximate version for the case where random sub-matrices are ill-conditioned.
\begin{lemma}
\label{lem:find_basis_approx_low_rank}
Suppose dataset $\dataset\in \R^{N\times d}$ satisfies the property that $\max_{\vw\in \normset{S}}\norm{\vw}_{\infty}>\sqrt{\kappa(\dataset)}$ with probability at least $\delta$ over $S\sim \binom{[d]}{k}$. Then, with probability at least $\Omega(\delta/k)$ over $S'\sim \binom{[d]}{k-1}$, it holds that there is a set $J$ with $|J|\geq \Omega(d\delta/k)$ such that for all $j\in J$
\begin{equation}
\label{eqn:approx_low_rank}
    \dataset^{(j)}= \vv(j)+c(j)\cdot \vec{z}(j)
\end{equation}
for $\vv(j)\in \colspan{\dataset_{S'}}$, $\frac{1}{N}\twonorm{\vec{z}(j)}^2\leq 1$ and $c(j)<\frac{1}{\sqrt{\kappa(\dataset)}}$. 
\end{lemma}
\begin{proof}
    Define the same random variable $(S',i)$ as in the proof of \Cref{lem:find_basis_random_low_rank}. From a similar averaging argument, we have that with probability at least $\delta/(2k)$ over $S'\sim \binom{[d]}{k-1}$, and probability at least $\delta/(2k)$ over $i\in [d]\setminus S'$, there is a vector $\vw(i)\in \normset{S'\cup \{i\}}$ with $|(\vw(i))_{i}|> \sqrt{\kappa(\dataset)}$. Now, define $J(S')$ as the set of indices in $[d]\setminus S'$ that satisfy the previous property. For all $i\in J$, define $\vec{z}(i)$ as the vector $\dataset\vw(i)$. Since $\vw(i)$ is in $\normset{S'\cup\{i\}}$, we have that $\frac{1}{N}\vec{z}(i)\leq 1$. Thus, we have that \[
    \dataset^{(j)}=\frac{1}{(\vw(i))_{i}}\Bigr(-\sum_{j\in S'} (\vw(i))_{j}\vec{X}^{(j)} + \vec{z}(i)\Bigr).
    \]
    The proof follows from the fact that $|(\vw(i))_{i}|>\sqrt{\kappa(\dataset)}$.
\end{proof}

The above lemma says that all the columns in $J$ can be approximated using the span of $S'$. Recall that we had satisfied the premise of \Cref{lem:find_basis_approx_low_rank} because $(\vec{I}_{(d)},\emptyset)$ was not a good preconditioner. The lemma suggests a natural update to the preconditioner: let $I$ be the set $S'$ and $\vec{B}$ be the matrix formed by starting with the identity matrix and replacing the rows $j$ in $J$ with $\vw(j)^{\top}$. Equivalently, the matrix $\vec{Z}$ obtained from $\dataset$ after the change of basis has $\vec{Z}^{(j)}=\vec{z}(j)$ for all $j\in J$ and has the same columns as $\dataset$ outside $J$. We show that this update to the preconditioner has moved us closer to our goal of finding a $(k^2/\delta,k,\delta,\sqrt{\kappa(\dataset)})$-good preconditioner. Concretely, we show that 
\begin{claim}
\label{claim:fixed_indices}
     For any set $S\in \binom{[d]}{k}$ and $\vw\in \normset{S}$, we have that 
    \[
        \max_{j\in J} |(\vec{B}^{-\top}\vw)_{j}|<\sqrt{\kappa(\dataset)}.
    \]
\end{claim}
\begin{proof}[Proof sketch]
After the above update, note that $|\vec{B}_{jj}|>\sqrt{\kappa(\dataset)}$ for all $j\in J$ (by construction). Now, we use a structural property of $\vec{B}$ following from Property (2) of \Cref{defn:good_change_of_basis}: $(\vec{B}^{-1})_{jj}=\frac{1}{\vec{B}_{jj}}$ and $(\vec{B}^{-1})_{ij}=0$ for $j\notin I$ and $i\neq j$ (see \Cref{lem:preconditioner_properties}). By definition of $\normset{S}$ and \Cref{defn:k_sparse_cd}, we have that $\norm{\vw}_{\infty}\leq \kappa(\dataset)$. Thus, we have that $(\vec{B}^{-\top}\vw)_{j}=\frac{\vw_{j}}{\vec{B}_{jj}}< \sqrt{\kappa(\dataset)}$ for all $j\in J$.
\end{proof}

The above claim tells us that the preconditioner $(\vec{B},I)$ improved the norm of the ground truth predictor (over the new basis) such that for all indices in $J$, the corresponding coefficients are less than $\sqrt{\kappa(\dataset)}$. Thus, we have essentially fixed the issue of large coefficients for an $\Omega(\delta/k)$-fraction of the indices. Thus, starting from the assumption that $(\vec{I}_{(d)},\emptyset)$ was not a good preconditioner, we can construct a new preconditioner $(\vec{B},I)$ with $|I|\leq k-1$ that effectively resolved the issue of large coefficients for a non-trivial fraction ($\Omega(\delta/k)$) of the indices outside $I$. 

\paragraph{Improving all coordinates} To continue the win-win argument, we will repeat this idea. We will show that starting with any preconditioner that is not good, we can construct a new preconditioner that adds $k-1$ indices to the set $I$ such that the new preconditioner fixes the issue of large coefficients in a new $\Omega(\delta/k)$ fraction of the indices. A key property we use is that every improvement step fixes the issue of large coefficients in a disjoint set of coordinates (the set $J$ found in each step is disjoint from the previous ones). This is shown using \Cref{claim:fixed_indices}. Thus, after repeating this idea $O(k/\delta)$ times, we would have either fixed all $d$ indices, or ended up with a good preconditioner that did not satisfy the premise of the improvement step. Either way, our final preconditioner after these $O(k/\delta)$ iterations is $(O(k^2/\delta),k,\delta,\sqrt{\kappa(\dataset)})$-good. We refer the reader to \Cref{thm:improve_norm_single_step} for more details.

\paragraph{Implementation} We note that, so far, we have omitted some aspects of computational efficiency. In particular, we did not specify how to find the set $S'$, $J$ and $\vec{z}(j)$ for $j\in J$ that are guaranteed by \Cref{lem:find_basis_approx_low_rank}. Finding these are crucial for actually implementing the preconditioner improvement steps we had described. Thankfully, the proof of \Cref{lem:find_basis_approx_low_rank} gives a natural algorithm: sample a random size $k-1$ subset $S'$ and for each index $j$ outside $S'$, attempt to find an appropriate vector $\vec{z}(j)$ satisfying \Cref{eqn:approx_low_rank}. We claim that this can be done in polynomial time using a convex program. Specifically, for each $j$ outside $S'$, we solve the following program: $\max_{\vw\in \normset{S'\cup \{j\}}} |\vw_{j}|$. Note that the set $\normset{S' \cup \{j\}}$ is a convex set. The maximization problem can be equivalently solved by minimizing $\vw_{j}$ and then taking absolute value, this is because $\normset{S}$ is symmetric about the origin. The minimization problem is a convex optimization problem with quadratic constraints and can be solved in polynomial time. We note that the success probability of this procedure (to find $S'$, J) would be $\Omega(\delta/k)$ as guaranteed by \Cref{lem:find_basis_approx_low_rank}. We boost this probability by repeating it $O(k/\delta)$ times.

\paragraph{Single Improvement Step: General Version} So far, we have sketched  our approach on how to find a $(O(k^2/\delta),k,\delta,\sqrt{\kappa(\dataset)})$-good preconditioner for the uniform distribution over the rows of $\dataset$. For the sake of going beyond $\sqrt{\kappa(\dataset)}$, we prove the following stronger theorem. We refer to \Cref{proof:improve_norm_single_step} for the proof.

\begin{theorem}[Theorem~\ref{thm:improve_norm_single_step}, restated]
    \label{thm:improve_norm_single_step_overview}
     Let $\dataset\in \R^{N\times d}$ be a dataset with $\max_{i\in [d]}\twonorm{\dataset^{(i)}}\leq \sqrt{N}$. Let $(\bin,\iin)$ be an $(\ell,k,\deltain,\kappain)$-good preconditioner for the uniform distribution over the rows of $\vec{X}$. Then, there exists an algorithm that runs in time $\poly(N,d,1/\delta,\log \kappain)$ and, with probability at least $0.99$ outputs a pair $(\bout,\iout)$ such that  $(\bout,\iout)$ is an $(\ell+O(k^2/\delta),k,\deltain+\delta,\sqrt{\kappain})$-good preconditioner for the uniform distribution over the rows of $\vec{X}$.
\end{theorem}

\paragraph{Iterative Improvement} The sketch we had presented so far was for the special case of the above theorem where $(\bin,\iin)=(\vec{I}_{(d)},\emptyset)$. The above theorem generalizes this idea: it provides an algorithm that given any $(\ell,k,\delta,\kappa)$-good preconditioner, efficiently computes an improved one. Specifically, it reduces the $\kappa$-parameter by a square-root (from $\kappain$ to $\sqrt{\kappain}$) at the cost of adding $O(k^2/\delta)$ to $\iin$.
\Cref{thm:improve_norm_single_step_overview} suggests an inductive approach to obtain a preconditioner with $\kappa=O(1)$: repeat the algorithm from \Cref{thm:improve_norm_single_step_overview} $\log \log \kappa(\dataset)$ times.
\begin{proof}[Proof sketch of \Cref{thm:find_good_basis}]
    We repeatedly apply the algorithm from \Cref{thm:improve_norm_single_step_overview}. Start with $B^{[0]}=\vec{I}_{(d)}$ and $I^{[0]}=\emptyset$. Inductively define $(\vec{B}^{[i]},I^{[i]})$ as the output of the algorithm from \Cref{thm:improve_norm_single_step_overview} on input $(\vec{B}^{[i-1]},I^{[i-1]})$. By choosing the $\delta$ parameter in \cref{thm:improve_norm_single_step_overview} as $\delta'$, it follows by induction that $(\vec{B}^{[t]},I^{[t]})$ is an $(O(k^2t/\delta'),k,t\delta',\kappa^{1/2^{t}})$-good preconditioner for the uniform distribution over the rows of $\dataset$. Choosing $t=\log\log \kappa(\dataset)$ and $\delta'=\delta/t$ completes the proof. 
\end{proof}

\section{Phase 1: Finding a Good Preconditioner}
\label{sec:good_precond_details}
We prove \Cref{thm:find_good_basis} in this section. We give an algorithm that progressively improves the quality of the preconditioner. We first analyze one step of the algorithm.

\begin{theorem}
    \label{thm:improve_norm_single_step}
    There exists an algorithm $\improvenorm$ with the following specifications. Let $\dataset\in \R^{N\times d}$ be a dataset with $\max_{i\in [d]}\twonorm{\dataset^{(i)}}\leq \sqrt{N}$ and let $\delta,\gamma$ be from $(0,1)$. Let $(\bin,\iin)$ be an $(\ell,k,\deltain,\kappain)$-good preconditioner for the uniform distribution over the rows of $\vec{X}$. $\improvenorm$, upon receiving $\dataset,\bin,\iin,\kappain,\delta$,$\gamma$ runs in time $\poly(N,d,1/\delta,\log \kappain,\log (1/\gamma))$ and, with probability at least $1-\gamma$, outputs a pair $(\bout,\iout)$ such that  $(\bout,\iout)$ is an $(\ell+O(k^2/\delta),k,\deltain+\delta,\sqrt{\kappain})$-good preconditioner for the uniform distribution over the rows of $\vec{X}$.
\end{theorem}

Once we have the above theorem, we are ready to prove \Cref{thm:find_good_basis}.

\begin{proof}[Proof of \Cref{thm:find_good_basis}]
    We construct the final preconditioner $(\vec{B},I)$ by repeatedly applying $\improvenorm$. Formally, we define a sequence of preconditioners \[(\vec{B}^{[i]},I^{[i]})=\improvenorm(\dataset,\vec{B}^{[i-1]},I^{[i-1]},\kappa^{[i-1]},\delta',0.99/(\log\log \kappa(\dataset)))\] with $\kappa^{[i]}=\sqrt{\kappa^{[i-1]}}$, $\vec{B}^{[0]}=\vec{I}_{(d)}, I^{[0]}=\emptyset$ and $\kappa^{[0]}=\kappa(\dataset)$. We will now show that $(\vec{B}^{[t]},I^{[t]})$ is a $\big(O({k^2 t}/{\delta'}),k,t\delta',\kappa^{[t]}\big)$-good preconditioner for the uniform distribution over the rows of $\dataset$ for any $t$ with probability at least $1-(0.99t/(\log\log \kappa(\dataset)))$. We will prove this by induction. The base case ($t=0$) is immediate.  Assume that the statement is true for $t-1$. We will now prove it for $t$. From \Cref{thm:improve_norm_single_step}, the inductive hypothesis, and the definition of $(\vec{B}^{[t]},I^{[t]})$, we immediately have that $(\vec{B}^{[t]},I^{[t]})$  is a $\big(O({k^2t}/{\delta'}),k,t\delta',\sqrt{\kappa^{[t-1]}}=\kappa^{[t]}\big)$-good preconditioner with probability $1-(0.99t/(\log \log (\kappa \dataset)))$. 

    To complete the proof, we set $t=\log\log \kappa(\dataset)$ and $\delta'=\delta/t=\delta/(\log \log \kappa(\dataset))$. Clearly, by definition of $\kappa^{[t]}$, we have that $\kappa^{[t]}=(\kappa(\dataset))^{\frac{1}{2^{t}}}\leq O(1)$ for $t=(\log\log \kappa(\dataset))$. Thus, we have that $(\vec{B},I)=(\vec{B}^{[t]},I^{[t]})$ is a $\big(O({k^2(\log \log\kappa(\dataset))^2}/{
    \delta}),k,\delta\big)$ for the uniform distribution over the rows of $\dataset$. The run time is immediate from the run time of $\improvenorm$ (\Cref{thm:improve_norm_single_step}).
\end{proof}

\begin{algorithm}[ht]
    \caption{$\mathsf{ImproveNorm}(\dataset,\bin,\iin,\kappain,k, \delta,\gamma)$}
    \label{alg:improve_norm_full}
    \DontPrintSemicolon
        \KwIn{Dataset $\dataset\in \R^{N\times d}$, invertible matrix $\bin\in \R^{d\times d}$, set of indices $\iin\subseteq[d]$, $\kappain > 0$, sparsity parameter $k\in\mathbb{N}$, failure probability over support parameter $\delta\in (0,1)$, failure probability of algorithm parameter $\gamma\in (0,1)$}
        \KwOut{A pair $(\bout,\iout)$ consisting of a matrix $\bout\in\R^{d\times d}$ and $\iout\subseteq[d]$}
        \BlankLine
        Let $C\ge 1$ be a sufficiently large universal constant;\;
        Set $I\gets\iin$ and $\vec{B}\gets \bin$;\;
        \For{$\ell\gets 1$ \KwTo $3k/\delta$}{
            \label{lineno:loop_outer}
        \For{$m\gets 1$ \KwTo $(Ck/\delta) \log(d/\gamma)$}{
        Sample a random subset $S'$ of $[d]$ of size $k-1$;\;
        \label{lineno:random_set_sample}
        Set $J\gets \emptyset$, 
        $\btemp\gets\vec{B}$, and
        $\itemp\gets I\cup S'$;\;
        \label{lineno:itemp}
        \For{$j \in [d]\setminus \itemp$}{
            \label{lineno:for_condition_itemp}
            Let ${\cal W}_j = \{\vv\in \R^d: \supp(\vv)\subseteq S'\cup\{j\}, \|\dataset \vv\|_2\le \sqrt{N}, |((\bin)^{-\top}\vv)_j| \le \kappain\}$;\label{lineno:constraints}\;
            Let $\vw$ be the solution to the following program:
            \begin{equation*}
                \begin{aligned}
                    \max\; & |(\vec{B}^{-\top}\vv)_j| \\
                    \text{s.t.}\; & \vv \in {\cal W}_j
                    \end{aligned}
            \end{equation*}\;
            \label{lineno:convex_program}
            \If{$|(\vec{B}^{-\top}{\vw})_j|> \sqrt{\kappain}$}{
                \label{lineno:condition_include_j}
                Set $J\gets J\cup \{j\}$ and
                \label{lineno:update_basis0}
                $\btemp_j\gets {\vw}^{\top}$;\;
            \label{lineno:update_basis1}

            }
        }
        \If{$|J|\geq {\delta d}/({3k})$}{
            \label{lineno:update_is_large}
            Set $\vec{B}\gets \btemp$,
            $I\gets \itemp$ and \textbf{break} from inner loop;\;
        }
        \label{lineno:update_is_large_end}
     
        }
        \If{{Inner loop did not change $(\vec{B},I)$}}{
        \label{lineno:break_outer}
        \textbf{break} from outer loop;\;
        }
        }
        Let $(\bout,\iout)\coloneq(\vec{B},I)$;\;
\end{algorithm}
\begin{remark}
    The inverses in \Cref{alg:improve_norm_full} do not need to be explicitly computed. It follows from \Cref{lem:preconditioner_properties} that  $(\vec{B}^{-\top}\vw)_j=\frac{\vw_j}{\vec{B}_{jj}}$ for $j\notin I$ when $(\vec{B},I)$ is a good preconditioner.
\end{remark}
\subsection{Analysis of \Cref{alg:improve_norm_full}}
In this section, we prove \Cref{thm:improve_norm_single_step}.
    Before proving the theorem, we outline how the algorithm works. 
    \begin{enumerate}
        \item Starting with $(\vec{B},I)\gets(\bin, \iin)$, the algorithm repeatedly samples random sets $S'$ of size $k-1$ and tests whether there are many coordinates $j\notin \iin \cup S'$ for which one can find a feasible $\vv$, supported on $S'\cup \{j\}$, with $|(\vec{B}^{-\top}\vv)_{j}|$ being large.
        \item For any such $j$ where this coefficient exceeds $\sqrt{\kappain}$, the algorithm replaces the $j$-th row of $\vec{B}$ with $\vv^{\top}$. We will show that this update reduces the magnitude of the $j$-th coordinate in the new basis. 
        \item If the number of such indices $j$ exceeds $\delta d/(3k)$, the algorithm adds $S'$ to $I$ and accepts the update; otherwise, it resamples $S'$ and repeats the previous steps. If this fails to produce an update after many repetitions, we conclude that $(\vec{B},I)$ is already a good preconditioner.
        \item The process continues for $O(k/\delta)$ rounds, each improving an additional $\Omega(\delta d/k)$ coordinates, or terminates when no further improvement is possible.
        \item The final output $(\bout,\iout)$ satisfies $|\iout|\leq k^2/\delta$ and we will show that it is a good preconditioner with the desired parameters.
    \end{enumerate}
   We prove that the output $(\bout,\iout)$ is a $(\ell,k,\deltain+\delta,\sqrt{\kappain})$-good preconditioner for the uniform distribution over the rows of $\dataset$. For a set $S\subseteq [d]$, recall that  $\normset{S}$ denotes the set of vectors $\vw\in \R^{d}$ such that $\supp(\vw)\subseteq S$ and $\frac{1}{N}\cdot \twonorm{\dataset\vw}^2\leq 1$. We will repeatedly use this notation in the proof. We prove that the three properties required by \Cref{defn:good_change_of_basis} are satisfied. 
   The following structural structural claim about preconditioners satisfying
\Cref{defn:good_change_of_basis} wil be useful throughout the section.

\begin{claim}
    \label{lem:preconditioner_properties}
    Let $(\vec{B},I)$ be a $(\ell,k,\delta)$-good preconditioner for any distribution $D$. Then, the following hold:
    \begin{enumerate}
        \item $|\supp(\vec{B}_i)|\leq \ell+1$ for all $i\in [d]$.
        \item $(\vec{B}^{-1})_{ii}=1/(\vec{B}_{ii})$ for all $i\notin I$.
    \end{enumerate}
\end{claim}
\begin{proof}
    Without loss of generality, assume $I=[\ell]$ (this can be achieved by permuting rows and columns). Thus, from \Cref{defn:good_change_of_basis} (2), we have that 
  \[
        \vec{B}^{-1}=\begin{pmatrix}
\vec{A} & \vec{0} \\
\vec{C} & \vec{D}
\end{pmatrix}
\]
    where $\vec{A}\in \R^{\ell\times \ell}$ is invertible and $\vec{D}$ is an invertible diagonal matrix. Now, on applying the inverse, we have that 
    \begin{equation*}
        \vec{B}=\begin{pmatrix}
\vec{A}^{-1} & \vec{0} \\
-\vec{D}^{-1}\vec{C}\vec{A}^{-1} & \vec{D}^{-1}
\end{pmatrix}
    \end{equation*}
    The conclusions of the lemma follow by inspecting the above matrix. 
\end{proof}
We are now ready to start proving the correctness of \Cref{alg:improve_norm_full}. First, we show that Property (1) is satisfied.
   \begin{lemma}[{Property 1}]
   \label{lem:prop1}
       For all $i\in [d]$, it holds that
       \[
       \frac{1}{N}\cdot \twonorm{\dataset^{\top}(\bout)_{j}}^2\leq 1
       \]
   \end{lemma}
   \begin{proof}
  The second constraint in the definition of ${\cal W}_j$ in \Cref{lineno:constraints}, together with the fact that we started with a good preconditioner guarantees that \[\frac{1}{N}\cdot \twonorm{\dataset^{\top}(\btemp)_i}^2\leq 1\] for all $i\in [d]$ throughout the algorithm. Since $\btemp$ is the running variable that is ultimately output, property (1) of \Cref{defn:good_change_of_basis} holds for the output of the algorithm.
\end{proof}
Next, we show that Property (2) is satisfied.
 \begin{lemma}[Property 2]
 \label{lem:prop2}
     For all $i\in [d]$, it holds that 
     \[
     \supp(((\bout)^{-1})_i)\subseteq I\cup \{i\}.
     \]Furthermore, we have that $|\iout|\leq |\iin|+ O(k^2/\delta)$.
 \end{lemma}
 \begin{proof}
   We show that $\bout$ is invertible with inverse satisfying $\supp(((\bout)^{-1})_i)\subseteq \iout\cup \{i\}$ for all $i\in [d]$.  We will do so by showing that $\btemp$ is invertible throughout the run of the algorithm. We describe the construction of $(\btemp)^{-1}$ and show that the following invariant holds throughout the algorithm: 
   \begin{equation}
       \supp((\btemp)^{-1})_i)=\{i\}\cup \itemp\,,\;\text{ for all }\;i\in [d] \tag{$\inv$}\label{equation:invariant-for-property-2}
   \end{equation}
   Clearly, invariant \eqref{equation:invariant-for-property-2} holds in the beginning with $(\btemp)^{-1}=(\bin)^{-1}$ since $\btemp=\bin$ and $\itemp=\iin$.
   Note that $\btemp$ is updated in \Cref{lineno:update_basis1}. Assume that the invariant \eqref{equation:invariant-for-property-2} holds before \Cref{lineno:update_basis1}. 
   
   We now describe how the inverse $(\btemp)^{-1}$ changes when $\btemp$ is updated according to \Cref{lineno:update_basis1}. Let $j\in J$ be the row of $\btemp$ that is being updated. Observe that $j\notin \itemp$ by \Cref{lineno:for_condition_itemp}. For any $\vec{x}\in \R^{d}$, let $\ztemp(\vec{x})\coloneq \btemp\vec{x}$ (before the update). 

   We claim that after modifying row $j$ of $(\btemp)$, it suffices to update only the corresponding row $j$ of $(\btemp)^{-1}$. 
To see this, observe that before the update we have
\[
((\btemp)^{-1})_{ij} = 0 \quad \text{for all } i \neq j,
\]
as guaranteed by the invariant~\eqref{equation:invariant-for-property-2} (since $j \notin \itemp$). 
That is, the $j$-th column of $(\btemp)^{-1}$ has all zeros except possibly at its $j$-th entry.
Now consider the effect of changing row $j$ of $(\btemp)$ on the product $(\btemp)^{-1} (\btemp)$. 
For any vector $\vec{x} \in \mathbb{R}^d$,
\[
((\btemp)^{-1} (\btemp) \vec{x})_i = \sum_{t} ((\btemp)^{-1})_{it} ((\btemp) \vec{x})_t.
\]
Since $((\btemp)^{-1})_{ij} = 0$ for all $i \neq j$, modifying the $j$-th row of $(\btemp)$ can only affect the term corresponding to $i = j$. 
Hence, $((\btemp)^{-1} (\btemp) \vec{x})_i$ remains unchanged for all $i \neq j$, meaning that the action of $(\btemp)^{-1}$ on $(\btemp)$ (and thus its inverse relationship) is preserved in all but the $j$-th coordinate. 
Therefore, to restore the invariant after the update, it is sufficient to change only the $j$-th row of $(\btemp)^{-1}$; all other rows remain valid.
   
   We now describe the change to row $j$ of the inverse. We change the $j$-th row of $(\btemp)^{-1}$ such that $(\btemp)^{-1}\btemp \vec{x}=\vec{x}$ for all $\vec{x}\in \R^{d}$. Let $Y= (\vw\cdot \vec{x}) =\btemp_j\cdot \vec{x}$ (after the update). Let $S'=(i_1,\ldots, i_{k-1})$ be the set chosen in \Cref{lineno:random_set_sample}. We have that 
   \begin{align*}
    \vw_j\vec{x}_j&={Y}-\sum_{t\in[k-1]}\vw_{i_t}\vec{x}_{i_t}\\
    &={Y}-\sum_{p\in \itemp}\alpha_p\ztemp(\vec{x})_p\\
   \end{align*}
   for appropriate choice of $\alpha_k$. This is because $\{i_t\}_{t\in[k-1]}\subseteq \itemp$ (\Cref{lineno:itemp}) and $\supp(((\btemp)^{-1})_i)\subseteq \{i\}\cup\itemp$ for all $i$ (before the update). 
   Thus, we have that $\vec{x}_j=\frac{1}{\vw_j}\cdot {Y}-\sum_{p\in \itemp} \beta_p\ztemp(\vec{x})_p$ for appropriate $\beta_p$. Since we update $\btemp$ as $\btemp_j\gets \vw^{\top}$, we have that $\ztemp_j$ is updated to $Y$. Thus, we can update $(\btemp)^{-1}$ such that $((\btemp)^{-1})_{jj}\gets\frac{1}{\vw_j}$, $((\btemp)^{-1})_{jp}\gets\beta_{p}$ for $p\in \itemp$ and $0$ otherwise. 
   
   Note that $(\alpha)_{p\in [d]}$ and $(\beta)_{p\in[d]}$ in the previous step were not dependent on $\vec{x}$. They were only dependent on the vector $\vec{w}$ and $\btemp$ before the update. Thus, we have constructed an inverse $(\btemp)^{-1}$ such that $(\btemp)^{-1}\btemp \vec{x}=\vec{x}$ for all $\vec{x}\in \R^{d}$. We also have that $\supp(((\btemp)^{-1})_{i})\subseteq \itemp \cup \{i\}$ for all $i\in [d]$, by construction. 
   
   Thus, we have that $(\btemp)^{-1},\btemp$ and $\itemp$ satisfy the invariant \eqref{equation:invariant-for-property-2} after the update in \Cref{lineno:update_basis1}. Since $\btemp$ is used to update $\vec{B}$, we have that the output of the algorithm also satisfies the invariant \eqref{equation:invariant-for-property-2}. 
   
   Finally, we  argue that $|\iout|\leq \ell+O(k^2/\delta)$. To see this, observe that each run of the outer loop (starting on \Cref{lineno:loop_outer}) adds at most $k$ variables to $\iin$ (\Cref{lineno:itemp}) and the outer loop runs $O(k/\delta)$ times.
    \end{proof}
    Finally, we show that $(\bout,\iout)$ satisfy property (3) of \Cref{defn:good_change_of_basis}. Formally, we show the following lemma.
    \begin{lemma}[Property 3]
    \label{lem:prop3}
        With probability at least $1-\delta$ over $S\sim \binom{[d]}{k}$, for all $\vw\in \normset{S}$, it holds that
        \[
        \norm{((\bout)^{-\top}\vw)_{\overline{\iout}}}_{\infty}\leq \sqrt{\kappain}.
        \]
    \end{lemma}
     Before proving \Cref{lem:prop3}, we prove some intermediate claims and lemmas.
    The first is an expression for $(\vec{B}^{-\top}\vw)_{j}$ when $j\notin I$. This will be repeatedly used in the proof.
    \begin{claim}
        \label{claim:coeff_j}
            For any pair $(\vec{B},I)$ such that $\supp((\vec{B}^{-1})_i)\subseteq I\cup \{i\}$ for all $i\in [d]$ and $\vw\in \R^{d}$, it holds that $(\vec{B}^{-\top}\vw)_j=\frac{\vw_j}{\vec{B}_{jj}}$ for all $j\notin I$.
        \end{claim}
        \begin{proof}
        Since $\supp((\vec{B}^{-1})_j)\subseteq I\cup\{j\}$, we have that $(\vec{B}^{-\top}\vw)_j=(\vec{B}^{-1})_{jj}\cdot \vw_{j}$ for $j\notin I$. The claim follows by applying \Cref{lem:preconditioner_properties}.
        \end{proof}
   Next, we show that updating row $\vec{B}_j$ improves the quality of the preconditioner if we only care about the $j$-th coordinate.
   \begin{claim}
   \label{claim:update_good}
   For any $j\in [d]$ for which the row $\vec{B}_j$ is updated during the course of the algorithm, the following holds after the update: for all $\widetilde{\vw}\in \normset{[d]}$ with $|((\bin)^{-\top}\widetilde{\vw})_j|\leq \kappain$, it holds that $|(\vec{B}^{-\top}\widetilde{\vw})_{j}|\leq \sqrt{\kappain}$. 
    \end{claim}
    \begin{proof}
      Note that the only place row $\vec{B}_j$ is possibly updated is in \Cref{lineno:update_basis1}. 
      Say we update $\vec{B}_{j}\gets \vw^{\top}$ for some sparse $\vw\in \normset{[d]}$ with $|((\bin)^{-\top}\vw)_{j}|\leq \kappain$ and $|((\btemp)^{-\top}\vw)_{j}|>\sqrt{\kappain}$. Before updating $\vec{B}_j$ for the first time, we have $(\vec{B}^{-1})_{jj}=((\btemp)^{-1})_{jj}=((\bin)^{-1})_{jj}$. From the premise about $\vw$ and \Cref{claim:coeff_j}, we have that \begin{equation}
      \label{eqn:premise_update}|((\bin)^{-1})_{jj}|\geq \frac{\sqrt{\kappain}}{|\vw_{j}|}.
      \end{equation} After the update, we have $(\vec{B}^{-1})_{jj}=\frac{1}{\vec{B}_{jj}}=\frac{1}{\vw_j}$. Now consider any vector $\widetilde{\vw}\in \R^{d}$ which satisfies the premise of the claim. We have that \begin{align*}
    |(\vec{B}^{-\top}\widetilde{\vw})_{j}|&=\frac{|\widetilde{\vw}_j|}{|\vw_j|}\leq \frac{|((\bin)^{-1})_{jj}\widetilde{\vw}_j|}{|((\bin)^{-1})_{jj}\vw_{j}|}\\
    &\leq \frac{\kappain}{\sqrt{\kappain}}\leq \sqrt{\kappain}.
      \end{align*}
    The first equality follows from \Cref{claim:coeff_j}. The penultimate inequality follows from \Cref{eqn:premise_update} and the fact that $|((\bin)^{-1})_{jj}\widetilde{\vw}_{j}|=|((\bin)^{-\top}\widetilde{\vw})_{j}|\leq \kappain$. This completes the proof of the claim.
    \end{proof}
    Our final intermediate lemma will be crucial in showing that whenever $(\vec{B},I)$ fails to satisfy property (3) of \Cref{defn:good_change_of_basis}, we can find a large set of indices to improve. This claim can be seen as a natural generalization of \Cref{lem:find_basis_approx_low_rank} from \Cref{sec:good_precond_overview}.
    \begin{lemma}  
        \label{claim:inf_norm_probalistic}
        Let $\vec{B}\in \R^{d\times d}$ and $\iin \subseteq I\subseteq [d]$ be such that $\supp((\vec{B}^{-1})_i)\subseteq \{i\} \cup I$ for all $i\in [d]$. Assume that with probability at least $(\deltain+\delta)$ over $S\sim \binom{[d]}{k}$, there exists $\vw\in \normset{S}$ such that  $\norm{(\vec{B}^{-\top}\vw)_{\overline{I}}}_{\infty}>\sqrt{\kappain}$. Then, with probability at least $\frac{\delta}{2k}$ over $S'\sim \binom{[d]}{k-1}$, it holds that 
        \begin{equation}
        \label{eqn:good_event}
            \Pr_{j\sim [d]\setminus S'}\Bigr[\exists \vw\in \normset{S'\cup \{j\}}, j\in \overline{I}, |((\bin)^{-\top}{\vw})_{i}|\leq \kappain \text{ and } |(\vec{B}^{-\top}{\vw})_{j}|>\sqrt{\kappain}\Bigr]\geq \delta/(2k)
        \end{equation}
    \end{lemma}

    \begin{proof}
    Recall that $\normset{S}$ was the set defined as $\{\vw\in \R^{d}\mid \twonorm{\dataset\vw}\leq \sqrt{N} \text{ and } \supp(\vw)\subseteq S\}$.
        From the premise of lemma and the fact that $(\bin,\iin)$ is an $(\ell,k,\deltain,\kappain)$-good preconditioner, we have  that with probability at least $\delta$ over $S\sim \binom{[d]}{k}$ , there exists some $\vw\in \normset{S}$ such that (1) $\|{((\bin)^{-\top}\vw)_{\overline{\iin}}\|}_{\infty}\leq \kappain$ and (2) $\|{(\vec{B}^{-\top}\vw)_{\overline{I}}\|}_{\infty}>\sqrt{\kappain}$. This follows from a union bound: the first condition holds for all $\vw\in \normset{S}$ with probability at least $1-\deltain$ for random $S$ and the second happens for some $\vw\in \normset{S}$ with probability at least $\deltain+\delta$. 
        
        Since $\supp((\vec{B}^{-1})_i)\subseteq I\cup \{i\}$ for all $i\in [d]$, we have from \Cref{claim:coeff_j} that the non-zero indices of $\vec{B}^{-\top}\vw$ are contained in $S$ . Thus, we have that with probability at least $\delta/k$ over $S\sim \binom{[d]}{k}$ and $j\sim S$,  there exists $\vw\in \normset{S}$ such that $j\in \overline{I}$, $|((\bin)^{-\top}\vw)_j|\leq \kappain$ and $|(\vec{B}^{-\top}\vw)_j|>\sqrt{\kappain}$. 
        This is because a random index from $S$ witnesses the large infinity norm with probability at least $\delta/|S|\geq \delta/k$. 
        We now give an alternate way to sample $(S,j$). First sample $S'\sim \binom{[d]}{k-1}$ and then sample $j\sim [d]\setminus S'$ and form $(S,i)\coloneq (S'\cup \{j\},j)$. Clearly these two sampling techniques result in the same distribution. We will now work with the second one. By a standard averaging argument, we have that with probability at least $\delta/(2k)$ over $S'\sim \binom{[d]}{k-1}$, it holds that \begin{equation*}
            \Pr_{j\sim [d]\setminus S'}\Bigr[\exists \vw\in \normset{S'\cup \{j\}}, j\in \overline{I}, |((\bin)^{-\top}{\vw})_{j}|\leq \kappain \text{ and } |(\vec{B}^{-\top}{\vw})_{j}|>\sqrt{\kappain}\Bigr]\geq \delta/(2k).
        \end{equation*}  This completes the proof. 
    \end{proof}

    We are now ready to complete the proof of \Cref{lem:prop3}.
    \begin{proof}[Proof of \Cref{lem:prop3}]
    From \Cref{claim:update_good}, we have that every row $\vec{B}_j$ can be updated at most once during the run of the algorithm (\Cref{lineno:condition_include_j} is violated after the update). Thus, all the indices that are updated in distinct iterations of the outer loop of the algorithm are distinct. This means at most $3k/\delta$ iterations of the outer loop are possible as each iteration updates $\delta d/(3k)$ distinct indices and there are a total of $d$ indices. This justifies the choice of the number of iterations in the outer loop of \Cref{alg:improve_norm_full}.

    We now argue that the output of the algorithm $(\bout,\iout)$ is a good preconditioner with desired parameters. Note that there are two ways the algorithm can terminate: (1) all $d$ indices where updated, or (2) some iteration of the inner loop had no changes (\Cref{lineno:break_outer}). We show that both of these imply with high probability that the output of the algorithm is a good preconditioner. To do this, we will use \Cref{claim:inf_norm_probalistic}.

    The pair $(\vec{B},I)$ (from the algorithm) satisfies the assumptions of the \Cref{claim:inf_norm_probalistic} at the start of the inner loop. We now argue that \Cref{claim:inf_norm_probalistic} implies that we will find a large set $J$ to update in \Cref{lineno:update_is_large} with probability at least $\delta/2k$ over the random set sampled in \Cref{lineno:random_set_sample}. The proof is almost immediate from \Cref{claim:inf_norm_probalistic} and the fact that $j$ is sampled uniformly from $[d]\setminus S'$. We have that each $j\in [d]\setminus S'$ has probability exactly $1/(d-k+1)$ of being picked.  Thus, for the event from \Cref{eqn:good_event} of \Cref{claim:inf_norm_probalistic} to hold with probability at least $\delta/(2k)$, there must be at least $(\delta (d-k+1))/(2k)\geq \delta d/(3k)$ indices $j\in [d]\setminus S\cap \overline{I}$ that satisfy the event from \Cref{eqn:good_event}. Thus, the size of $J$ in \Cref{lineno:update_is_large} is at least $(\delta d)/(3k)$ with probability at least $\delta/2k$ over the choice of $S'\sim \binom{[d]}{k-1}$. 
    
    Repeating the inner loop $(10k/\delta)\log (d/\gamma)$ times guarantees that a large set will be found with probability at least $1-\gamma/d$ in each iteration of the inner while loop where the assumption holds. Thus, a union bound over at most $d$ iterations of the outer while loop makes the final error probability of the algorithm $\gamma$. 
    Recall again the two ways the algorithm could terminate: (1) all $d$ indices where updated, or (2) some iteration of the inner loop had no changes. 
    
    In the case of (1), the final pair $(\vec{B},I)$ cannot satisfy the premise of \Cref{claim:inf_norm_probalistic} as there are no indices left to update (recall that every index is updated at most one). Since we constructed $(\vec{B},I)$ to satisfy the condition about the support of the rows of $(\vec{B}^{-1})$, the only way they can fail the premise of \Cref{claim:inf_norm_probalistic} is if the probability of sampling $S\sim \binom{[d]}{k}$ for which there exists $\vw\in \normset{S}$ with $\norm{((\vec{B}^{-\top})\vw)_{\overline{I}}}_{\infty}>\sqrt{\kappain}$ is at most $(\deltain+\delta)$. This implies that $(\vec{B},I)$ satisfies the conclusion of \Cref{lem:prop3}.

    In the case of (2), we have that with probability at most $1-\gamma$, the premise of \Cref{claim:inf_norm_probalistic} is not satisfied. Thus, repeating the argument from the last paragraph, we again have that $(\vec{B},I)$ satisfies the conclusion of \Cref{lem:prop3}.
    \end{proof}

    We are ready to complete the proof of \Cref{thm:improve_norm_single_step}.

    \begin{proof}[Proof of \Cref{thm:improve_norm_single_step}]
\label{proof:improve_norm_single_step}
Combining \Cref{lem:prop1,lem:prop2,lem:prop3}, we have that $(\bout,\iout)$ is a $(\ell+O(k^2/\delta),k,\deltain+\delta,\sqrt{\kappain})$-good preconditioner.
    
    We finally analyze the runtime of the algorithm. The most computationally expensive step of the algorithm is the program in \Cref{lineno:convex_program}. We claim that this program can be solved in time $\poly(N,d,\log (\kappain))$. Although the program is not convex as written, it can be solved using one. Observe that the constraints are symmetric with respect to $\vw$, that is they are satisfied by both $\vw$ and $-\vw$. Consider the convex program obtained from the program in \Cref{lineno:convex_program} by replacing the objective with $\min (\vec{B}^{-\top}\vw)_j$. Clearly, the absolute value of the solution of this convex program  will give the solution to the program from \Cref{lineno:convex_program}. Also, each entry of $\vec{B}$ is at most $\kappain$ (from \Cref{lineno:update_basis1}) and hence can be represented with $O(\log \kappain)$ bits. This convex program has a linear objective and quadratic constraints and hence can be solved in time $\poly(d,N,\log \kappain)$. Thus, the total run time is $\poly(N,d,1/\delta,\log (1/\gamma))$ where we also account for the iterations of the two loops. 
\end{proof}

\section{Phase 2: Solving SLR after finding a good preconditioner}
\label{sec:low_error_after_precond}
In this section, we prove \Cref{thm:good_basis_implies_low_error}. To do this, we first  prove a general theorem that states that running regression with an $\ell_1$ constraint only on the set $\bar{I}$ has sample complexity scaling with the size of $I$ and the $\ell_1$ norm outside $I$.
\begin{theorem}
\label{thm:partial_lasso}
    Let $\vec{X}\in \R^{N\times d}$ be a matrix with $\max_{i\in [d]}\norm{\vec{X}^{(i)}}_2\leq \sqrt{N}$. Let $I\subseteq [d]$ be a subset of indices. Let $\vw^{*}\in \R^{d}$ such that $\norm{\vw^{*}_{\bar{I}}}_1\leq B$,  and $\frac{1}{N}\norm{\vec{X}\vw^*}_2^2\leq 1$. Let $\labels=\vec{X}\vw^{*}+\noise$ where $\noise$ is a mean zero $\sigma$-sub-gaussian random vector that is independent of $\dataset$. Then the solution $\widehat{\vw}$ to the program
\begin{equation}
\begin{aligned}
\min_{\vw \in \mathbb{R}^d} \; & \frac{1}{N}\,\norm{\labels - \vec{X}\vw}_{2}^{2} \\
\text{s.t.}\; & \norm{\vw_{\bar I}}_1 \le B, \\
              & \frac{1}{N}\norm{\vec{X}\vw}_2^2 \le 1.
\end{aligned}
\end{equation} satisfies the equation 
\begin{equation}
    \frac{1}{N}\norm{\vec{X}\widehat{\vw}-\vec{X}\vw^*}_2^2\leq \frac{(B+1)\sigma}{\sqrt{N}}\cdot O\big(\sqrt{|I|}+\sqrt{\log d}+\sqrt{\log (1/\gamma)}\big) 
\end{equation} with probability $1-\gamma$ over the noise.
\end{theorem}
\begin{proof}
    We have that $\vw^{*}$ is feasible from the assumptions of the claim. From the definition of $\widehat{\vw}$, we have that 
    \[
    \frac{1}{N}\twonorm{\labels-\vec{X}\widehat{\vw}}^2\leq \frac{1}{N}\twonorm{\labels-\vec{X}\vw^*}^2.
    \]
    On rearranging terms we obtain that 
    \begin{equation}
    \label{eqn:lasso_gen_0}
    \frac{1}{N}\twonorm{\vec{X}\widehat{\vw}-\vec{X}\vw^*}^2\leq \frac{2}{N}\noise^{\top}\vec{X}(\widehat{\vw}-\vw^*).
    \end{equation}
    For any set $S\subseteq [d]$, let $\vec{X}_S$ denote the matrix formed by the columns of $X$ indexed by entries in $S$. To bound the right hand side of the above equation, it will be useful to bound $\twonorm{\vec{X}_I\vw_I}$ for any feasible $\vw$. Consider any feasible $\vw$. Note that $\vec{X}\vw=(\vec{X}_I\vw_I)+(\vec{X}_{\bar{I}}\vw_{\bar{I}})$. We have that $\norm{\vec{X}_{\bar{I}}\vw_{\bar{I}}}_2\leq \sum_{i\in \bar{I}}|\vw_i|\norm{\vec{X}^{(i)}}_2\leq B\max_{i\in [d]}\norm{\vec{X}^{(i)}}_2\leq B\sqrt{N}$. By triangle inequality, we have that $\twonorm{\vec{X}_I\vw_I}\leq \twonorm{\vec{X}\vw}+\twonorm{\vec{X}_{\bar{I}}\vw_{\bar{I}}}\leq (B+1)\sqrt{N}$. 
    We are now ready to bound right hand side of \Cref{eqn:lasso_gen_0}. We have that 
    \begin{align}
    \label{eqn:lasso_gen_1}
        \noise^{\top} \vec{X}(\widehat{\vw}-\vw^*)&\leq |\noise^{\top} \vec{X}_I(\widehat{\vw}_I-\vw^*_I)|+|\noise^{\top}\vec{X}_{\bar{I}}(\widehat{\vw}_{\bar{I}}-\vw^*_{\bar{I}})|\nonumber\\
        &\leq \sup_{\twonorm{\vec{X}_I\vw_I}\leq 2(B+1)\sqrt{N}}|\noise^{\top}\vec{X}_I\vw_I|+\sup_{\norm{\vec{u}}_1\leq 2B}|\noise^{\top}\vec{X}_{\bar{I}}\vec{u}|
    \end{align}
    The last inequality follows the fact that $\twonorm{\vec{X}_I\vw_I}\leq (B+1)\sqrt{N}$ and $\norm{\vw_{\bar{I}}}_1\leq B$ for any feasible $\vw$. We now bound the two terms on the right hand side separately. First, we have that  with probability at least $1-\gamma$ over $\noise$, it holds that
    \begin{align*}
        \sup_{\twonorm{\vec{X}_I\vec{w}_I}\leq 2(B+1)\sqrt{N}}|\noise^{\top}\vec{X}_I\vw_I|&\leq 2(B+1)\sqrt{N}\sup_{\substack{\vec{u}\in \mathrm{span}(X_I)\\\twonorm{\vec{u}}=1}}|\noise^{\top}\vec{u}|\\
        &\leq 2(B+1)\sqrt{N}\cdot \twonorm{\vec{P} \noise}\\
        &\leq 2(B+1)\sqrt{N}\cdot O\big(\sigma\sqrt{I}+\sigma\sqrt{\log (1/\gamma)}\big)
    \end{align*}
    where the $\vec{P}$ is the orthogonal projection matrix onto $\colspan{\dataset_{I}}$. The last inequality follows from the facts that (1) $\vec{P}\noise$ is a mean zero sub-gaussian random vector supported on an $|I|$ dimensional subspace and (2) concentration of norm of a sub-gaussian vector around its mean.

For the second term in \Cref{eqn:lasso_gen_1}, we have that with probability at least $1-\gamma$ over $\noise$
\begin{align*}
    \sup_{\norm{\vec{u}}_1\leq 2B}|\noise^{\top}\vec{X}_{\bar{I}}\vec{u}|&\leq 2\norm{ \vec{X}_{\bar{I}}^{\top}\noise}_{\infty}B\\
    &\leq 2B\cdot O\big(\sigma\sqrt{N\log d}+\sigma\sqrt{N\log(1/\gamma)}\big)
\end{align*}

Combining the previous two displays with \Cref{eqn:lasso_gen_0,eqn:lasso_gen_1}, we have that 
\begin{equation}
    \frac{1}{N}\twonorm{\dataset\wopt-\vec{X}\widehat{\vw}}^2\leq \frac{(B+1)\sigma}{\sqrt{N}}\cdot O\big(\sqrt{|I|}+\sqrt{\log d}+\sqrt{\log (1/\gamma)}\big)
\end{equation}
\end{proof}

We are now ready to prove \Cref{thm:good_basis_implies_low_error}.

\begin{proof}[Proof of \Cref{thm:good_basis_implies_low_error}]
    \label{proof:good_basis_implies_low_error}
    The algorithm proceeds as follows: (1) construct new matrix $\vec{Z}=\dataset\vec{B}^{\top}$, (2) solve the regression task with $\vec{Z}$ as the features to obtain a hypothesis $\widehat{\vu}$, (3) output hypothesis $\widehat{\vw}=\vec{B}^{\top}\widehat{\vu}$.

    Consider $\wopt$ from the theorem statement. Define $\uopt\colon \vec{B}^{-\top}\wopt$. Clearly, we have that $\vec{Z}\uopt=\dataset\wopt$. Now, from the \Cref{defn:good_change_of_basis} and the fact that $\frac{1}{N}\twonorm{\dataset \wopt}^2\leq 1$, we have that with probability at least $1-\delta$ over the support of $\wopt$, it holds that $\norm{\uopt_{\overline{I}}}_{\infty}\leq O(1)$. Also, from \Cref{claim:coeff_j}, we have that $\norm{\uopt_{\overline{I}}}_1\leq O(k)$. Hereafter, we will only analyze $\wopt$ for which the previous property holds (it holds with probability at least $1-\delta$ over the support). 

    We now use the algorithm from \Cref{thm:partial_lasso} to implement step (2) of our algorithm. We have that $\norm{\uopt_{\overline{I}}}_1\leq O(k)$. From property (1) of \Cref{defn:good_change_of_basis}, we have that $\max_{i\in [d]}\twonorm{\vec{Z}^{(i)}}\leq \sqrt{N}$. We have that ${\vec{Z}}$ and $\vu^*$ satisfy the premise of \Cref{thm:partial_lasso}. Thus, on running the convex program in \Cref{thm:partial_lasso}, we obtain a vector $\widehat{\vu}$ such that 
    \[
    \frac{1}{N}\twonorm{{\vec{Z}}\widehat{\vu}-{\vec{Z}}\vu^*}^2\leq \frac{k\sigma}{\sqrt{N}}\cdot O(\sqrt{\ell}+\sqrt{\log d})
    \]
    where we used the fact that $|I|=\ell$. Let $\widehat{\vw}=\vec{B}^{\top}\widehat{\vu}$. Thus, we have that 
    \begin{align*}
    \frac{1}{N}\twonorm{\dataset\widehat{\vw}-\dataset\wopt}^2&=    \frac{1}{N}\twonorm{\vec{Z}\vec{B}^{-\top}\widehat{\vw}-\vec{Z}\vec{B}^{-\top}\wopt}^2\\   
    &=\frac{1}{N}\twonorm{\vec{Z}\widehat{\vu}-\vec{Z}{\vu^*}}^2\\
    &\leq \frac{k\sigma}{\sqrt{N}}\cdot O(\sqrt{\ell}+\sqrt{\log d})
    \end{align*}
This concludes the proof of \Cref{thm:partial_lasso}.
\end{proof}
\section{The Gaussian Design Setting}

In this section, we prove \Cref{thm:low_population_error} which shows how to go from low training error to low error over the population when the input dataset $\dataset$ consists of $N$ independent rows drawn from the distribution $\Gauss(\vec{0},\vsigma)$ for some unknown matrix $\vsigma$. 

\subsection{Finding a Good Preconditioner} 

The first step of the proof is to find a good preconditioner with respect to $\Gauss(\vec{0},\vsigma)$, where $\vsigma$ is a matrix whose condition number $\kappa(\Sigma)$ (or an upper bound thereof) is known, but the matrix is otherwise unknown. In particular, we prove the following theorem.

\begin{theorem}
\label{thm:find_good_basis_distributional}
    Let $\vsigma\in \R^{d\times d}$ be a positive semi-definite matrix with $\max_{i\in [d]}\vsigma_{ii}\leq 1$ for all $i\in [d]$. There is an algorithm that draws $N= O(k\log d)$ independent samples from $\Gauss(\vec{0},\vsigma)$, runs in time $\poly(N,d,1/\delta,\log \kappa(\vsigma))$ and with probability at least $0.9$ outputs an invertible matrix $\vec{B}\in \R^{d\times d}$ and set $I\subseteq [d]$ such that $(\vec{B},I)$ is a $\big(O({k^2(\log \log\kappa(\vsigma))^2}/{
    \delta}),k,\delta\big)$-good preconditioner for $\Gauss(\vec{0},\vsigma)$. 
\end{theorem}
In order to prove \Cref{thm:find_good_basis_distributional}, we use the algorithm from \Cref{thm:find_good_basis} to do this, combined with the following standard matrix concentration result.

\begin{theorem}[Loewner Concentration, see e.g. \cite{vershynin2018high}]\label{theorem:loewner-concentration}
    Let $\vsigma\in \R^{d\times d}$ be a symmetric positive semi-definite matrix and let $\dataset\in\R^{N\times d}$ be a matrix whose rows are i.i.d. draws from $\Gauss(\vec{0},\vsigma)$ with $N \ge C \frac{d+\log(1/\gamma)}{\eps^2}$ for some sufficiently large universal constant $C\ge 1$. Then, with probability at least $1-\gamma$, the following is true:
    \[
        (1-\eps) \|\vsigma^{1/2} \vw\|_2^2 \le \frac{1}{N}\cdot \|\dataset \vw\|_2^2 \le (1+\eps)  \|\vsigma^{1/2} \vw\|_2^2 \,, \text{ uniformly over all }\vw\in\R^d.
    \]
\end{theorem}

We are now ready to prove \Cref{thm:find_good_basis_distributional}.

\begin{proof}[Proof of \Cref{thm:find_good_basis_distributional}]
Let $\dataset\in \R^{N\times d}$ be a matrix where each row is an independent sample from $\Gauss(\vec{0},\vsigma)$. From our choice of $N$ and \Cref{theorem:loewner-concentration}, we have that with probability at least $0.99$, for all $\vec{\vw\in \R^{d}}$ with $|\supp(\vec{\vw})|\leq k$, it holds that 
\begin{equation}
\label{eqn:mult_conc_sparse}
\frac{1}{2}\twonorm{\vsigma^{\frac{1}{2}}\vw}^2\leq \frac{1}{N}\cdot \twonorm{\dataset\vw}^2\leq 2 \twonorm{\vsigma^{\frac{1}{2}}\vw}^2.
\end{equation}
Here, \Cref{theorem:loewner-concentration} is applied to the principle submatrix of $\vsigma$ corresponding to the set $\supp(\vw)$. Henceforth, we will assume that \Cref{eqn:mult_conc_sparse} holds for $\dataset$. Since $\max_{i\in [d]}\vsigma_{ii}\leq 1$ and from \Cref{eqn:mult_conc_sparse}, the premise of \Cref{thm:find_good_basis} applies to the matrix $\frac{1}{2}\cdot \dataset$. Also, from \Cref{defn:k_sparse_cd} and \Cref{eqn:mult_conc_sparse}, we have that $\kappa(\dataset)\leq 2 \cdot \kappa(\vsigma)$. Thus, on running the algorithm from \Cref{thm:find_good_basis} on input $\frac{1}{2}\cdot \dataset$, we obtain a pair $(\vec{B},I)$ such that the following hold:
\begin{enumerate}
    \item \label{item:norm} $\max_{i\in [d]}\frac{1}{N}\twonorm{\dataset^{\top}\vec{B}_{i}}^2\leq 1$.
   
    \item  \label{item:set}$|I|\leq \ell$ and $\supp((\vec{B}^{-1})_i)\subseteq I\cup \{i\}$ for all $i\in [d]$.
   
    \item \label{item:low_norm}With probability $1-\delta$ over $S\sim \binom{[d]}{k}$, for all $\vec{w}\in \R^{d}$ with $\supp(\vw)\subseteq S$ and $\frac{1}{N}\twonorm{\dataset\vec{w}}^2\leq 1$, it holds that \[\norm{(\vec{B}^{-\top}\vw)_{\overline{I}}}_{\infty}\leq C\]  
    
\end{enumerate}
where $C$ is a universal constant and $\ell=O(k^2(\log\log \kappa(\vsigma))^2/\delta)$. We now use one additional observation to show that $(\frac{1}{2}\cdot \vec{B},I)$ is an $(\ell,k,\delta)$-good preconditioner for $\Gauss(\vec{0},\vsigma)$. 
\begin{claim}
    \label{claim:sparse_rows_preconditioner}
    The rows of the matrix $\vec{B}$ output by the algorithm from \Cref{thm:find_good_basis} are $k$ sparse. 
\end{claim}
\begin{proof}
    The proof is almost immediate from inspection of the algorithm from \Cref{thm:find_good_basis} (\Cref{alg:improve_norm_full}). The only place where the rows of $\vec{B}$ are set are in \Cref{lineno:update_basis1} and they are set to be equal to sparse vectors $\vw$ by definition. 
\end{proof}

Thus, from \Cref{claim:sparse_rows_preconditioner}, \Cref{eqn:mult_conc_sparse} and \Cref{item:norm}, we have that \[\E_{\vec{x}\sim \Gauss(\vec{0},\vsigma)}[(\vec{B}_i\cdot \vec{x})^{2}]=\twonorm{\vsigma^{\frac{1}{2}}\vec{B}_i}^2\leq \frac{2}{N}\twonorm{\dataset^{\top}\vec{B}_i}^2\leq 2\] for all $i\in [d]$. This proves property (1) of \Cref{defn:good_change_of_basis}. Property (2) of \Cref{defn:good_change_of_basis} follows immediately from \Cref{item:set}. Finally, we prove property (3) from \Cref{defn:good_change_of_basis}. From \Cref{eqn:mult_conc_sparse}, we have that \Cref{item:low_norm} is true even with the constraint $\frac{1}{N}\twonorm{\dataset\vw}^2\leq 1$ replaced with the constraint $\twonorm{\vsigma^{\frac{1}{2}}\vw}^2\leq \frac{1}{2}$. Thus, with probability $1-\delta$ over $S\sim \binom{[d]}{k}$, for all $\vw\in \R^{d}$ with $\supp(\vw)\subseteq S$ and $\twonorm{\vsigma^{\frac{1}{2}}\vw}^2\leq 1$, it holds that \[\norm{(\vec{B}^{-\top}\vw)_{\overline{I}}}_{\infty}\leq 2C.\] Thus, we have that $(\frac{1}{2}\cdot \vec{B},I)$ is a $O(k^2((\log\log \kappa(\vsigma))^2)/\delta,k,\delta)$-good preconditioner for $\Gauss(\vec{0},\vsigma)$. The run-time of the algorithm is immediate given the run-time of \Cref{thm:improve_norm_single_step}.
\end{proof}

\subsection{Generalization Bound} 

The second tool we use for the proof of \Cref{thm:low_population_error} is the following generalization result for sparse linear regression when both the ground truth $\vw^*$ and the candidate output $\vw$ have bounded $\ell_1$ norm outside some fixed set coordinates of bounded size. This structural property is ensured by running the regression algorithm of \Cref{thm:partial_lasso} after obtaining a good preconditioner through \Cref{thm:find_good_basis_distributional}.

\begin{theorem}\label{theorem:generalization}
    Let $\vsigma\in \R^{d\times d}$ be a covariance matrix such that $\vsigma_{ii} \le v$ for all $i\in[d]$ and let $I\subseteq[d]$ with $|I|\le \ell$ and $\vw^*\in\R^d$ such that $\|\vw^*_{\bar I}\|_1 \le B$ and $\|\vsigma^{1/2}\vw^*\|_2 \le r$. Then, with probability at least $1-\gamma$ over the random choice of a matrix $\dataset\in \R^{N\times d}$ whose rows are i.i.d. draws from $\Gauss(\vec{0},\vsigma)$, where $N \ge C\frac{(r + (v+1) B)^2}{\eps^2} (\ell+\log(d/\gamma))$ for some sufficiently large constant $C\ge 1$, the following is true:
    \[
        \|\vsigma^{1/2} (\vw-\vw^*)\|_2^2 \le (1+\eps)\frac{1}{{N}}\|\dataset (\vw-\vw^*)\|_2^2 + \eps\,,
    \]
    uniformly over all $\vw\in \R^d$ such that $\|\vw_{\bar I}\|_1 \le B$ and $\frac{1}{\sqrt{N}}\|\dataset \vw\|_2 \le r$.
\end{theorem}

To prove the above theorem, we use \Cref{theorem:loewner-concentration}, which we introduced earlier, as well as the following versatile generalization result from \cite{zhou2024optimistic}, applied to our setting.

\begin{theorem}[\cite{zhou2024optimistic}]\label{theorem:optimistic}
    Let $\vsigma \in \R^{d\times d}$ be symmetric positive semi-definite, $F: \R^d \to [0,\infty)$, $\gamma\in (0,1)$, $\vw^*\in\R^d$ and $C,C'\ge 1$ sufficiently large universal constants. Suppose $N\ge C\log(1/\gamma)$ and
    \[
        \pr_{\x\sim \Gauss(\vec{0},\vsigma)}[\forall \vw\in\R^d: (\vw-\vw^*)\cdot \x \le F(\vw)] \ge 1-\gamma\,.
    \]
    Then, with probability at least $1-2\gamma$ over the random choice of a matrix $\dataset\in \R^{N\times d}$ whose rows are i.i.d. draws from $\Gauss(\vec{0},\vsigma)$, the following is true uniformly over all $\vw\in \R^d$:
    \[
        \|\vsigma^{1/2} (\vw-\vw^*)\|_2^2 \le \Bigr( 1 + C'\sqrt{{\log(1/\gamma)}/{N}}\Bigr)  \Bigr( \frac{1}{\sqrt{N}}\|\dataset (\vw-\vw^*)\|_2 + \frac{1}{\sqrt{N}} F(\vw) \Bigr)^2
    \]
\end{theorem}

We are now ready to prove \Cref{theorem:generalization}.

\begin{proof}[Proof of \Cref{theorem:generalization}]
    We will apply \Cref{theorem:optimistic}. To this end, we would like to bound the inner product $(\vw-\vw^*)\cdot \x$ by some function $F(\vw)$, with high probability over $\x\sim \Gauss(\vec{0},\vsigma)$. We have $(\vw-\vw^*)\cdot \x = (\vw_I-\vw^*_I)\cdot \x + (\vw_{\bar I}-\vw^*_{\bar I})\cdot \x$. Let $\z \sim \Gauss(0,\mathbf{I})$ so that $\x = \vsigma^{1/2}\z$. For the first term, we have the following:
    \begin{align*}
        (\vw_I-\vw^*_I)\cdot \x = \z \cdot (\vsigma^{1/2} (\vw_I-\vw^*_I)) \le \|\vsigma^{1/2}(\vw_I-\vw^*_I)\|_2 \sup_{\vu\in W_I: \|\vu \|_2 = 1} |\z \cdot \vu|\,,
    \end{align*}
    where $W_I$ is the row space of the matrix $\vsigma^{1/2}\mathbf{P}$, where $\mathbf{P}$ is the orthogonal projection matrix onto the subspace corresponding to the coordinates in $I$ (i.e., $\mathbf{P}$ is a diagonal matrix with ones in $i$-th positions for $i\in I$ and zero otherwise). Observe that the dimension of $W_I$ is at most $|I| \le \ell$. Therefore, with probability at least $1-\delta/4$, we have $\sup_{\vu\in W_I: \|\vu \|_2 = 1} |\z \cdot \vu| \le 2\sqrt{\ell}+ 4\sqrt{ \log(4/\gamma)}$, due to Chi-squared concentration.

    Moreover, we have $\|\vsigma^{1/2}(\vw_I- \vw^*_I)\|_2 \le \|\vsigma^{1/2}\vw_I\|_2+\|\vsigma^{1/2} \vw^*_I\|_2$ and also, due to \Cref{theorem:loewner-concentration} applied to the principal submatrix of $\vsigma$ corresponding to $I$, and since $N \ge C(\ell + \log(4/\gamma))$ we have the following with probability at least $1-\gamma/4$:
    \begin{align*}
        \|\vsigma^{1/2} \vw_I\|_2 &\le \frac{2}{\sqrt{N}}\|\dataset \vw_I\|_2 \le \frac{2}{\sqrt{N}}\|\dataset \vw\|_2 + \frac{2}{\sqrt{N}}\|\dataset_{\overline{I}} \vw_{\bar I}\|_2 \\
        &\le 2\, r +  \frac{2}{\sqrt{N}}\Bigr\| \sum_{i\in\bar I} \vw_i \dataset^{(i)} \Bigr\|_2 \\
        &\le 2\, r +  \frac{2}{\sqrt{N}} \sum_{i\in\bar I} |\vw_i|\cdot  \|\dataset^{(i)} \|_2 \\
        &\le 2\, r +  \frac{2}{\sqrt{N}} \max_{i\in\bar I} \|\dataset^{(i)} \|_2 \|\vw_{\bar I}\|_1 \\
        &\le 2\, r +  4\, v\, B\,,
    \end{align*}
    where the last inequality follows from \Cref{theorem:loewner-concentration} (applied to the principal submatrix $\vsigma_{ii}\le v$), as well as the bound on $\|\vw_{\bar I}\|_1$. We obtain the same bound for $\|\vsigma^{1/2}\vw^*_I\|_2$ analogously. Hence, overall, we have the following with probability at least $1-\gamma/2$:
    \begin{equation*}
        (\vw_I-\vw^*_I)\cdot \x \le (4\, r + 4 v \, B + 4 v \|\vw_{\bar I}\|_1) \cdot (2\sqrt{\ell}+ 4\sqrt{ \log(4/\gamma)})
    \end{equation*}
    It remains to bound the second term $(\vw_{\bar I}-\vw^*_{\bar I})\cdot \x$. The following holds with probability at least $1-\gamma/2$ over $\x\sim \Gauss(\vec{0},\vsigma)$:
    \begin{align*}
        (\vw_{\bar I}-\vw^*_{\bar I})\cdot \x \le \|\x\|_\infty (B + \|\vw_{\bar I}\|) \le (B + \|\vw_{\bar I}\|) \cdot \sqrt{2\log (2d/ \gamma)}
    \end{align*}

    Therefore, we may now apply \Cref{theorem:optimistic} with $F(\vw) = O(r + v B + v\|\vw_{\bar I}\|_1) \cdot (\sqrt{\ell}+\sqrt{\log(d/\gamma)})$, which implies the desired result for the choice $N = C\frac{(r + v B)^2}{\eps^2} (\ell+\log(d/\gamma))$.
\end{proof}

\subsection{Putting Everything Together} 

In order to complete the proof of \Cref{thm:low_population_error}, we combine \Cref{thm:find_good_basis_distributional} with \Cref{thm:partial_lasso} and \Cref{theorem:generalization} to bound the population prediction error. 
\begin{proof}[Proof of \Cref{thm:low_population_error}]
    First draw $N_1=O(k\log d)$ independent samples from $\Gauss(\vec{0},\vsigma)$ and learn a $(\ell,k,\delta)$-good preconditioner for $\Gauss(\vec{0},\vsigma)$ with $\ell=O(k^2(\log \log \kappa(\vsigma))^2/\delta)$ using the algorithm from \Cref{thm:find_good_basis_distributional}.

    Define $\uopt\coloneq \vec{B}^{-\top}\wopt$. Applying property (3) from \Cref{defn:good_change_of_basis}, we have that with probability at least $1-\delta$ over the support of $\wopt$, it holds that $\|{\uopt_{\overline{I}}}\|_{\infty}\leq O(1)$. Also, from \Cref{claim:coeff_j}, we have that $\|{\uopt_{\overline{I}}}\|_1\leq O(k)$. Hereafter, we will assume that $\wopt$ satisfies this property (as we only need success probability at least $1-\delta$ over the support of $\wopt$). 

    Let $\widetilde{\vsigma}$ be the matrix $\vec{B}\vsigma \vec{B}^{\top}$. Clearly, for random variable $\vec{x}\sim \Gauss(\vec{0},\vsigma)$, the random variable $\vec{z}=\vec{B}\vec{x}$ is distributed as $\Gauss(0,\widetilde{\vsigma})$. From the properties of a good preconditioner, we have that $\max_{i\in [d]}\widetilde{\vsigma}_{ii}\leq 1$. Recall that the algorithm has sample access to $\slr(\vsigma,\wopt,\sigma)$. By pre-multiplying the covariates by $\vec{B}$, we can get sample access to $\slr(\widetilde{\vsigma},\uopt,\sigma)$. Draw $N=\Omega(k^4(\log \log \kappa(\vsigma))^2\log (d)/(\epsilon^2\delta))$ samples from $\slr(\widetilde{\vsigma},\uopt,\sigma)$ and represent them as a matrix $\vec{Z}\in \R^{N\times d}$ and vector $\labels\in \R^{N}$ where the columns of $\vec{Z}$ correspond to the independent samples from $\Gauss(0,\widetilde{\vsigma})$ and the entries $\labels$ are the labels of the $N$ samples of $\slr(\widetilde{\vsigma},\uopt,\sigma)$. Using \Cref{theorem:loewner-concentration} and from our choice of $N$, we have that $\max_{i\in [d]}\twonorm{\vec{Z}^{(i)}}\leq 2\sqrt{N}$ for all $i\in [d]$ and $\twonorm{\vec{Z}\uopt}\leq 2\sqrt{N}$. Thus, after rescaling by a factor of $2$, we can apply \Cref{thm:partial_lasso} with $B=O(k)$ and $|I|=(k^2/\delta)\cdot (\log \log \kappa(\vsigma))^2$. Thus, we have that as long as $N\geq \Omega(\sigma^2 k^4(\log \log \kappa(\vsigma))^2\log (d)/(\epsilon^2\delta))$, the program from \Cref{thm:partial_lasso} outputs a vector $\widehat{\vu}$ with $\norm{\widehat{\vu}_{\overline{I}}}_{1}\leq O(k)$ such that \[
    \frac{1}{N}\twonorm{\vec{Z}(\uopt-\widehat{\vu})}^2\leq \epsilon.
    \]

    Now, we want to argue that this generalizes to the population. Using \Cref{theorem:generalization} and the fact that (1) $N\geq\Omega(k^4(\log \log \kappa(\vsigma))^2\log (d)/(\epsilon^2\delta))$ and (2) $\max(\norm{\widehat{\vu}_{\overline{I}}}_{1},\norm{{\uopt}_{\overline{I}}}_{1})\leq O(k)$, we have that 
    \[
    \twonorm{\widetilde{\vsigma}^{\frac{1}{2}}(\uopt-\widehat{\vu})}^2\leq 3\epsilon.
    \]
    Now, taking $\widehat{\vw}\coloneq \vec{B}^{\top}\widehat{\vu}$ and substituting $\widetilde{\vsigma}=\vec{B}\vsigma\vec{B}^{\top}$, $\wopt=\vec{B}^{\top}\uopt$, we have that 
    \[
    \twonorm{\vsigma^{\frac{1}{2}}(\wopt-\widehat{\vw})}^2\leq 3\epsilon.
    \]This completes the proof. 
\end{proof}
\section{Lower Bounds for Good Preconditioners}
In this section, we construct a covariance matrix~$\vsigma$ such that any $(\ell, k, \delta)$-good preconditioner for $\Gauss(\vec{0}, \vsigma)$ must satisfy $\ell \ge \Omega(k^2 / \delta)$. This demonstrates that our algorithms in \Cref{thm:find_good_basis,thm:find_good_basis_distributional} are nearly tight, up to a $(\log \log \kappa(\vec{X}))^2$ factor. This lower bound is incomparable to, and not implied by, the one in~\cite{kkmr_lb}, which rules out the existence of preconditioners satisfying stronger properties. Our result is fine-grained, isolating the dependence on $k$ and~$\delta$; note that by \Cref{thm:find_good_basis}, a good preconditioner exists for every design matrix.
\begin{theorem}
\label{thm:good_precond_lower_bound}
    For any integers $\ell,d\geq 1$, $\delta\in (0,1)$ and $k\leq \sqrt{\delta d}$, there exists a positive definite matrix $\vsigma\in \R^{d\times d}$ with $\max_{i\in [d]}\vsigma_{ii}\leq 1$ such that for any constant $c>0$, an $(\ell,k,\delta, (\kappa(\vsigma))^{1-c})$-good preconditioner for $\Gauss(\vec{0},\vsigma)$ must have $\ell\geq \Omega(k^2/\delta)$. 
\end{theorem}
\begin{proof}
    Let $\alpha$ be a parameter to be chosen later (we will choose it so that $1/\alpha$ is an integer). Let $Z_1,\ldots Z_t$ for $t=1/\alpha$ be $t$ independent Gaussians sampled from $\Gauss(0,1)$. For $\epsilon>0$, define the variable $X_{i,j}=\sqrt{1-\epsilon}\cdot Z_{i}+\sqrt{\epsilon} \cdot Y_{i,j}$ for $i\in t$, $j\in \alpha d$, where $\{Y_{i,j}\}_{i\in [t],j\in [\alpha d]}$ are independent standard Gaussians. Consider the $d$ dimensional distribution defined by the random variables $\{X_{i,j}\}_{i\in [t],j\in [\alpha d]}$. Clearly, this distribution is equivalent to $\Gauss(\vec{0},\vsigma)$ where $\vsigma$ is a block diagonal matrix of the form 
\[
\vsigma =
\begin{pmatrix}
(1-\epsilon)\vec{J}_{(\alpha d)}+\epsilon \vec{I}_{(\alpha d)} & 0 & \cdots & 0\\
0 & (1-\epsilon)\vec{J}_{(\alpha d)}+\epsilon \vec{I}_{(\alpha d)} & \cdots & 0\\
\vdots & \vdots & \ddots & \vdots\\
0 & 0 & \cdots & (1-\epsilon)\vec{J}_{(\alpha d)}+\epsilon \vec{I}_{(\alpha d)}
\end{pmatrix}.
\]

We will show that there is no $(\ell,k,\delta, (\kappa(\vsigma))^{1-c})$-good preconditioner for $\Gauss(\vec{0},\vsigma)$ with $\ell< 1/(2\alpha)$. We will show this by contradiction. Suppose $(\vec{B},I)$ is an $(1/(2\alpha),k,\delta, (\kappa(\vsigma))^{1-c})$-good preconditioner for $\Gauss(\vec{0},\vsigma)$. We will show that there are rows $j$ for which $\E_{\vec{x}\sim \Gauss(\vec{0},\vsigma)}[(\vec{B}_j \cdot \vec{x})^2]\geq \Omega(1/\epsilon)$, which will contradict property (1) of \Cref{defn:good_change_of_basis} (as $\epsilon$ can be arbitrarily small). 

It is easy to see that $\kappa(\vsigma)=\Theta(\sqrt{1/\epsilon})$. We will henceforth index into $[d]$ using tuples of the form $(i,j)$ where $i\in [1/\alpha]$ and $j\in [\alpha d]$. We say that $X_{i,j}$ and $X_{i,k}$ belong to the same block (they have the same first coordinate). Let $T\coloneq \{i\mid i\in [1/\alpha], \forall j\in [\alpha d], (i,j)\notin I\}$ be the set of blocks where no indices are chosen in $I$. Clearly, we have $|T|\geq 1/(2\alpha)$. Without loss of generality, let $T=[1/(2\alpha)]$. 

We will use a birthday paradox argument. It will be convenient for us to assume that the random set $S$ in property (3) of \Cref{defn:good_change_of_basis} is sampled with replacement. Note that the TV distance between sampling $S$ with replacement and sampling $S$ without replacement is at most $k^2/(2n)$. This is because the probability of collision given $k$ indices chosen at random from $[d]$ with replacements is at most $(\binom{k}{2}/d)\leq k^2/(2n)$. Furthermore, on conditioning on the event of no collision, the distribution is equivalent to sampling without replacement. By our choice of $k$, we can assume that we are sampling with replacement by paying a TV distance penalty of $\delta/2$. 

From the birthday paradox (see, for example, Section~5.1 \cite{mitzenmacher2017probability}), we have that the probability of $k$ random indices drawn with replacement from $[d]$ having at least two indices in the same block  is at least $1-e^{-C k^2/t}$ where $t=1/\alpha$ and $C$ is some large universal constant. Choosing $t= C'k^2/\delta$ for constant $C'$ implies that with probability at least $6\delta$, there is some block where there are two indices chosen. This implies that with probability $3\delta$, there is a collision in the blocks indexed by $T$. Now, by the TV distance argument from the previous paragraph, with probability at least $2\delta$ over random $S\sim \binom{[d]}{ k}$, there is a block in $T$ where two indices are chosen. 

From the previous paragraph and property (3) of \Cref{defn:good_change_of_basis}, we have that there exists at least one set $S$ with $|S|=k$ such that $(i,j)\in S$ and $(i,k)\in S$ for $i\in T$ and for all $\vw$ with $\supp(\vw)=S$, it holds that 
\begin{equation}
    \label{eqn:lowerbound_prop3}
\E_{\vec{x}\sim \Gauss(\vec{0},\vsigma)}[(\vw\cdot \vec{x})^2]\leq 1\implies \norm{\vec{B}^{-\top}\vec{w}_{\overline{I}}}_{\infty}\leq (\sqrt{1/\epsilon})^{1-c}
\end{equation}

Consider $\vw$ formed as follows: $\vw_{(i,j)}=\sqrt{\frac{1}{2\epsilon}}$, $\vw_{(i,k)}=-\sqrt{\frac{1}{2\epsilon}}$ and $0$ everywhere else. From the construction of $\vsigma$, we have that $\E_{\vec{x}\sim \Gauss(\vec{0},\vsigma)}[(\vw\cdot \vec{x})^2]=\E[(Y_{i,j})^2/2]+\E[(Y_{i,k})^2/2]=1$. 

Recall from \Cref{claim:coeff_j} that $((\vec{B}^{-\top})\vw)_{i}=\frac{\vw_{i}}{\vec{B}_{i,i}}$ for all $i\notin I$. Since $(i,j)$ and $(i,k)$ are not in $I$ (from definition of $T$), \Cref{claim:coeff_j} and \Cref{eqn:lowerbound_prop3} implies that $|\vec{B}_{(i,j),(i,j)}|\geq \Omega((1/\epsilon)^{c/2})$ and $|\vec{B}_{(i,k),(i,k)}|\geq \Omega((1/\epsilon)^{c/2})$. Recall from \Cref{lem:preconditioner_properties} that $\supp(\vec{B}_{(i,j)})\subseteq I\cup \{(i,j)\}$. Thus, since $(i,j)\notin I$ and the fact that distinct blocks have independent origin centered Gaussians, we have that 
\[
\E_{\vec{x}\sim \Gauss(\vec{0},\vsigma)}[((\vec{B}_{(i,j)})\cdot \vec{x})^2]\geq |\vec{B}_{(i,j),(i,j)}|^2\geq \Omega((1/\epsilon)^c).
\]

This contradicts property (1) of \Cref{defn:good_change_of_basis} which requires the variance to be constant (as $\epsilon$ can be made arbitrarily small). Thus, we must have $\ell\geq 1/(2\alpha)\geq \Omega(k^2/\delta)$.
\end{proof}
\section{Conclusions}
We showed that sparse linear regression is tractable even for worst-case design matrices when the support of the sparse regression vector is chosen uniformly at random. Our techniques readily extend to any efficiently samplable support distribution with very high min-entropy. In particular, if the support distribution has min-entropy $(k-1)\log d+h$, then almost the same analysis yields sample complexity
 \[
 \poly(\log d,\log \log \kappa(\dataset),1/\delta,1/\epsilon)\cdot (d/2^{h}).
 \]
 This is non-trivial for any $h=\omega(1)$. Note that $h=\Omega(\log d)$ in the uniform distribution case. A natural open question is to determine whether similar results can be established under much weaker assumptions on the support distribution.

Another open question is to improve the dependence on $\epsilon$ in our sample complexity. Due to the use of the ``slow rate'' analysis of the Lasso, our current sample complexity scales as $(1 / \epsilon^2)$. A natural goal is to improve this dependence to $(1 / \epsilon)$.

\section*{Acknowledgements}
We thank Adam Klivans and Vasilis Kontonis for useful discussions during early stages of this project. We thank David Zuckerman for asking a question about the tightness of our result.

\newpage

\bibliographystyle{alpha}
\bibliography{refs}
\end{document}